\documentclass[10pt,journal,twoside]{IEEEtran}
\IEEEoverridecommandlockouts
\pdfoutput=1
\usepackage{hyperref}
\usepackage{cite}
\usepackage{url}
\usepackage{graphicx}
\usepackage[tbtags]{amsmath}
\usepackage{graphics}
\usepackage{epsfig}
\usepackage{amsthm, amssymb}
\usepackage{mathrsfs}
\usepackage{algpseudocode}
\usepackage[table,xcdraw]{xcolor}
\usepackage{array}
\usepackage{tabularx}

\newtheorem{theorem}{Theorem}

\usepackage{xspace}
\usepackage{subcaption}
\usepackage{makecell}
\usepackage{bbm}
\usepackage{booktabs}
\usepackage{amsthm}
\usepackage{multirow}
\newtheorem{assumption}{Assumption}

\renewcommand{\thesubfigure}{\alph{subfigure}}
\usepackage[linesnumbered,ruled]{algorithm2e}
\usepackage{csquotes}
\usepackage[]{caption}
\ifCLASSINFOpdf

\else
\fi
\begin{document}
\title{Fast AI Model Partition for Split Learning  over Edge Networks}
\author{Zuguang~Li,~\IEEEmembership{Graduate Student Member,~IEEE},
Wen~Wu,~\IEEEmembership{Senior~Member,~IEEE},\\
Shaohua~Wu,~\IEEEmembership{Member,~IEEE},
and
Xuemin~(Sherman)~Shen,~\IEEEmembership{Fellow,~IEEE}
\thanks{
Zuguang Li and Wen Wu are with the School of Electronics and Information Engineering, Harbin Institute of Technology, Shenzhen 518055, China, and also with the Department of Strategic and Advanced Interdisciplinary Research, Pengcheng Laboratory, Shenzhen, 518055, China (e-mail: lizuguang@stu.hit.edu.cn, wuw02@pcl.ac.cn). 

Shaohua Wu is with the School of Electronics and Information Engineering, Harbin Institute of Technology, Shenzhen 518055, China (e-mail: hitwush@hit.edu.cn).


Xuemin (Sherman) Shen is with the Department of Electrical and Computer Engineering, University of Waterloo, Waterloo, ON N2L 3G1, Canada (email: sshen@uwaterloo.ca).
}
}
\maketitle
\IEEEpubidadjcol

\begin{abstract}
Split learning (SL) is a distributed learning paradigm that can enable computation-intensive artificial intelligence (AI) applications by partitioning AI models between mobile devices and edge servers. 
However, the model partitioning problem in SL becomes challenging due to the diverse and complex architectures of AI models. In this paper, we formulate an optimal model partitioning problem to minimize training delay in SL. To solve the problem, we represent an arbitrary AI model as a directed acyclic graph (DAG), where the model's layers and inter-layer connections are mapped to vertices and edges, and training delays are captured as edge weights. Then, we propose a general model partitioning algorithm by transforming the problem into a minimum \textit{s-t} cut problem on the DAG. Theoretical analysis shows that the two problems are equivalent, such that the optimal model partition can be obtained via a maximum-flow method.
Furthermore, taking AI models with block structures into consideration, we design a low-complexity block-wise model partitioning algorithm to determine the optimal model partition. Specifically, the algorithm simplifies the DAG by abstracting each block (i.e., a repeating component comprising multiple layers in an AI model) into a single vertex. Extensive experimental results on a hardware testbed equipped with NVIDIA Jetson devices demonstrate that the proposed solution can reduce algorithm running time by up to 13.0$\times$ and training delay by up to 38.95\%, compared to state-of-the-art baselines.

	\vspace*{1mm}

\end{abstract}
\section{Introduction}
With the large-scale deployment of Internet of things (IoT) devices and the advancement of sensing technologies, massive amounts of data are continuously generated by mobile devices in edge networks. Leveraging such data, artificial intelligence (AI) can enable a wide range of intelligent applications in edge networks, including healthcare, autonomous driving, and smart cities~\cite{wang2021network}. Due to data privacy concerns, training AI models using traditional centralized learning is often difficult~\cite{Khaled2022edge}. Distributed learning frameworks, e.g., federated learning (FL), have been proposed to support collaborative model training without sharing raw data~\cite{shen2021holistic}. However, FL requires training the entire model on mobile devices, which is challenging, especially when their computing capabilities are limited.

Split learning (SL), as a promising distributed learning paradigm, has been proposed to address the weakness of FL~\cite{gupta2018distributed}. Specifically, SL divides the entire AI model into a device-side model and a server-side model. Each device trains only its device-side model using its own data and transmits intermediate results (i.e., smashed data and gradients) to the edge server instead of raw data~\cite{sun2024An}. In this way, SL can offload computational load from devices to the edge server while allowing devices to retain their raw data. In the SL framework, model partitioning is essential, particularly for edge networks, as it directly affects both communication overhead and device-side computational load.

{In the literature, several model partitioning methods have been proposed for SL to accelerate AI model training. In general, AI models can be classified into linear and non-linear architectures according to their layer structures~\cite{wang2022hivemind}. 
Linear AI models follow a strictly sequential structure, where each layer depends only on the output of the previous layer (e.g., LeNet~\cite{lecun1998gradient} and AlexNet~\cite{NIPS2012_c399862d}). 
In contrast, non-linear AI models contain complex connections such as skip connections or parallel branches (e.g., residual blocks in ResNet~\cite{targ2016resnet}), which create multiple dependency paths during computation.
For linear AI models, the number of available cuts is limited, and a brute-force search is a widely adopted approach to determine the optimal cut that minimizes model training latency~\cite{lin2024efficient, lim2024cutting}. For non-linear AI models, complex structures make brute-force search impractical for finding the optimal cut due to its prohibitive computational complexity. To determine the cut in non-linear AI models, a non-linear module was integrated into a single block, thereby transforming the non-linear model into a simplified linear structure for partitioning~\cite{wang2022hivemind}.
A regression-based method was designed to quantify the relationship between the cut and both the computing delay and communication delay, enabling the optimal cut to be obtained by solving an optimization problem~\cite{li2023Throughput}.
However, these approaches are generally suboptimal for non-linear AI models. Determining the optimal cut for arbitrary AI models faces three primary challenges: (\textit{i}) the AI models are highly complex, and accurately representing the relationship among layers is challenging; (\textit{ii}) during model training, device-side computation latency, server-side computation latency, and communication latency should be associated with the corresponding layers; and (\textit{iii}) an efficient method is required to quickly determine the optimal cut without incurring excessive computational latency.} Therefore, can we develop an efficient and generalizable approach to determine the optimal model partition for arbitrary AI models?


In this paper, we investigate the model partitioning problem for SL, aiming to quickly determine the optimal model partition that minimizes training delay in edge networks. To achieve this goal, \textit{firstly}, we represent an AI model as a directed acyclic graph (DAG), where the AI model’s layers and inter-layer connections are mapped to vertices and edges, respectively. To encode the computation and transmission delays in the DAG, we assign three types of edge weights: \textit{(i) device execution weight}, defined as the sum of the computation delay of training the corresponding layer on the device and the transmission delay for uploading that layer’s parameters; \textit{(ii) server execution weight}, defined as the sum of the computation delay of training the corresponding layer on the server and the transmission delay for downloading that layer’s parameters; and \textit{(iii) propagation weight}, defined as the transmission delay of the smashed data and the corresponding gradients propagated between two connected layers.
{\textit{Secondly}, for an arbitrary AI model, we propose a general model partitioning algorithm to determine the optimal cut. Specifically, the optimal model partitioning problem is reformulated as a minimum \textit{s–t} cut problem on the DAG. We further prove that these two problems are equivalent via rigorous theoretical analysis, meaning that solving the minimum \textit{s–t} cut directly yields the optimal model partition. Based on this equivalence, the optimal cut can be efficiently obtained using classical maximum-flow algorithms. \textit{Thirdly}, for an AI model with block structures, we propose a low-complexity block-wise model partitioning algorithm to determine the optimal cut. Specifically, the algorithm simplifies the DAG by abstracting each recurrent block (i.e., a module consisting of multiple layers) into a single vertex. In this way, the computational complexity of searching for the optimal cut can be further reduced.}
Our experimental results on a hardware testbed with NVIDIA Jetson devices demonstrate that the proposed solution can: {\textit{(i)} determine the optimal model partition for SL within milliseconds, where the block-wise model partitioning algorithm achieves up to a $13\times$ reduction in running time compared with the general model partitioning algorithm; and \textit{(ii)} accelerate model training by reducing the overall training delay by up to 38.95\% compared with the regression-based method.}
The main contributions of this paper can be summarized as follows:
\begin{itemize}
    \item {We represent an arbitrary AI model as a general DAG and formulate an optimal model partitioning problem to minimize the training delay in SL;}

    \item {We propose a general model partitioning algorithm to determine the optimal model partition, which equivalently transforms the optimal model partitioning problem into a minimum \textit{s-t} cut problem;}
    
    \item {We propose a block-wise model partitioning algorithm for AI models, which can quickly identify the optimal model partition.}

\end{itemize}

The remainder of this paper is organized as follows. Section~\ref{sec: related works} reviews the related works, followed by the system model in Section~\ref{sec: system model}. The DAG-based representation of AI models is presented in Section~\ref{sec: DAG-Based AI Model Depiction}. The general and block-wise model partitioning algorithms are presented in Sections~\ref{sec: DAG-based model partition Algorithm} and \ref{sec: block-wise model partition Algorithm}, respectively. Performance evaluation is provided in Section~\ref{sec: simulation results}. Finally, Section~\ref{sec: conclusion} concludes this paper.

\section{Related Work} \label{sec: related works}


{The SL paradigm has attracted widespread attention and investigation due to its ability to enable flexible computing load sharing across devices and edge servers for collaborative training.} In the 3rd generation partnership project (3GPP) Release 18, the performance advantages of AI model partition have been recognized, and protocol-level primitives have been explored to support model partition between end devices and the server~\cite{3gpp2021study}.
A communication-aware SL framework with an early-exit mechanism was developed to accommodate devices with diverse computational capabilities and fluctuating channel conditions~\cite{ninkovic2024comsplit}. An online model partition strategy was proposed to minimize energy and latency costs under dynamic wireless channel conditions~\cite{yan2022optimal}.
To support dynamic networks, a cluster-based outgoing SL framework was introduced, enabling a first-outgoing-then-sequential training paradigm for heterogeneous devices~\cite{wu2023split}.
Similarly, a parallel SL framework was proposed where all devices collaboratively train the device-side model with the server, addressing computational heterogeneity in edge networks~\cite{lin2024efficient}.
To further improve communication efficiency, a split federated learning framework was designed, where cut selection and bandwidth allocation are jointly optimized to minimize overall system latency~\cite{xu2023accelerating}. Existing work has explored SL to address challenges such as heterogeneous computing capabilities and dynamic channel conditions in edge networks. {In contrast, our work focuses on fast AI model partition for SL to reduce the computational complexity and the overall training delay.}



The selection of the cut layer in SL directly affects computational load and communication overload. A theoretical analysis of split federated learning was conducted to examine how the cut layer selection influences convergence and overall performance~\cite{dachille2024impact}. Another study modeled the cut layer selection process as a Stackelberg game to optimize client incentives and data contributions while balancing privacy preservation and energy consumption~\cite{lee2024game}.
To determine the optimal cut layer, various approaches have been proposed, including brute-force search~\cite{lin2024efficient}, binary search~\cite{duan2021joint}, regression-based search~\cite{wen2025training}, and graph-based search approaches~\cite{liang2023dnn, dai2024asurvey, wu2024pdd}. A brute-force search method was proposed to identify the optimal cut layer for either minimum latency or mobile energy consumption, based on profiling per-layer execution time and energy usage across different AI models~\cite{kang2017neurosurgeon}. A binary search-based partition algorithm was designed for pipeline inference in line-structured models~\cite{duan2021joint}. In graph-based approaches, a DAG-based method was proposed to select the optimal cut layer under dynamic network conditions, thereby reducing inference latency in edge–cloud systems~\cite{liang2023dnn}. Another approach addressed multi-partitioning in DAG-topology networks through graph transformation and partitioning techniques~\cite{wu2024pdd}. The existing model partition approaches are primarily designed for linear neural networks, limiting their applicability to more complex non-linear neural networks. Differently, our work focuses on finding the optimal model partition for arbitrary AI models.

\section{System Model} \label{sec: system model}
\subsection{Considered Scenario}
In this subsection, we present a typical SL scenario in an edge network and the model training process. As shown in Fig.~\ref{fig:system_model}, the system includes mobile devices and an edge server. The AI model is partitioned between the device and the server, with intermediate data transmitted via wireless channels. Key system components are described as follows:

\begin{itemize}
    \item  {\emph{Edge server:} The edge server is equipped with a base station that can perform server-side model training. Moreover, it is responsible for collecting network information, such as device computing capabilities and channel conditions, to support decision-making for the model partition.}
    
    \item {\emph{Mobile Devices:} The devices are endowed with computing capabilities, which can perform device-side model training. However, their computing capabilities are significantly weaker than the edge server’s, which motivates offloading part of the training workload to the server to accelerate the overall training process. They have local datasets and corresponding labels for model training.}
\end{itemize}

\begin{figure}[t]
	\renewcommand{\figurename}{Fig.}
	\centering
	\includegraphics[width=0.40\textwidth]{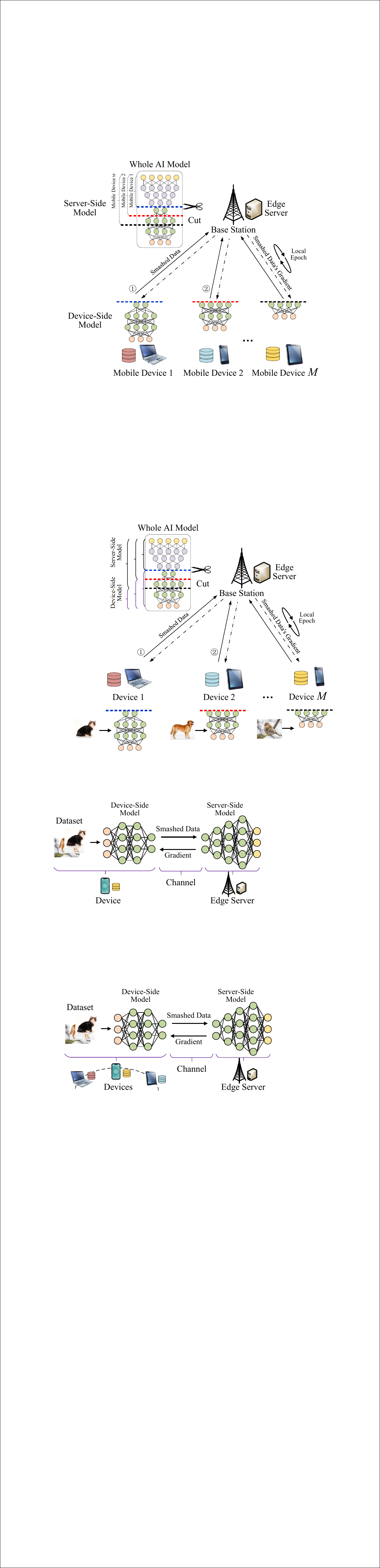}
	\caption{{Considered wireless SL scenario in edge networks.}}
	\label{fig:system_model}
\end{figure}

Considering the heterogeneous device capabilities and the dynamic channel conditions in edge networks, the server dynamically determines the optimal model partition for each device at each training epoch.
Before training begins, the model parameters are randomly initialized. In each training epoch, the server collects device and network information and selects a device for training in a round-robin manner. Based on the information, the server determines the model partition, denoted by ${c} = \{\mathcal{V}_D, \mathcal{V}_S\}$, where $\mathcal{V}_D$ and $\mathcal{V}_S$ denote the sets of layers assigned to the device and server, respectively. {The device and server then jointly train the AI model over $N_{loc}$ local iterations.}

{During each local iteration, the device performs forward propagation on the device-side model and transmits the smashed data and corresponding labels to the server. The server completes forward and backward propagation on the server-side model and returns the gradients to the device, which subsequently updates its device-side model. 
After completing local iterations, the device uploads its updated device-side model to the server, which integrates the updates and distributes the updated device-side model to the next selected device. This process is repeated across devices over multiple training epochs until the global model converges.}

\subsection{Training Delay}
In SL, the computation and transmission delays are analyzed as follows.

\subsubsection{Computation Delay}
{During each training epoch, the device-side model and the server-side model are trained at the device and server, respectively, whose delays are analyzed as follows:}
\begin{itemize}
    \item {\textit{Device-side model computation delay:} The computation delay in processing layer $v_i$ is defined as $\xi_{D, v_i}$. Note that $\xi_{D, v_i}$ consists of the forward and backward propagation computation delay at layer $v_i$. The device-side model computation delay can be defined as}
\begin{equation}\label{equ:Device-side model computation delay}
		T_{D,C} = \sum_{v_i \in \mathcal{V}_D} \xi_{D, v_i}.
\end{equation}

\item {\textit{Server-side model computation delay:} For the server, the computation delay in processing layer $v_i$ is defined as $\xi_{S, v_i}$. Hence, its computation delay can be defined as}
\begin{equation}  \label{server_comp_delay}
	T_{S,C} = \sum_{v_i \in \mathcal{V}_S } \xi_{S, v_i}.
\end{equation}
\end{itemize}

\subsubsection{Transmission Delay}
{During SL, the device-side model, labels, smashed data, and its gradients are exchanged between the device and server, incurring transmission delay. Since the transmission delay of labels is independent of the cut layer selection, it can be ignored.} The analysis of other transmission delays is as follows.

\begin{itemize}
    \item {\textit{Device-side model downloading delay:} Let $R_{S}$ denote the transmission rate from the server to the device, and $k_{v_i}$ denote the data size of the parameters in layer $v_i$.
Hence, the downloading delay for transmitting the device-side model can be defined as}
\begin{equation}
	T_{S,D}=\frac{\sum_{v_i \in \mathcal{V}_D} k_{v_i} }{R_{S}}.
\end{equation}

    \item {\textit{Smashed data and transmission delay:}  Let $R_{D}$ denote the average transmission rate from the device to the server, and $a_{v_i}$ denote the data size of the smashed data of vertex $v_i$ in the forward propagation.  The transmission delay for transmitting the smashed data can be defined as}
    \begin{equation}
        T_{D,S}=\frac{\sum_{v_i \in \mathcal{V}_c } a_{v_i}}{R_{D}},
    \end{equation}
    where $\mathcal{V}_c$ is the set of end layers on the device-side model after the cut splits the AI model.
    

    \item \textit{Smashed data's gradient transmission delay:} Let $\tilde{a}_{v_i}$ denote the size of the smashed data's gradient received by vertex $v_i$ during backward propagation. During AI model training, the size of smashed data is equal to that of its corresponding gradient, i.e., $\tilde{a}_{v_i} = {a}_{v_i}$.
    The transmission delay for transmitting the gradient can be defined as
    \begin{equation}
    	T_{S,G}=\frac{\sum_{v_i \in \mathcal{V}_c } \tilde{a}_{v_i}}{R_{S}}.
    \end{equation}

    \item \textit{Device-side model uploading delay:} The device-side model uploading delay can be defined as
    \begin{equation} \label{equ: device-side model upload delay}
            T_{D,U}=\frac{\sum_{v_i \in \mathcal{V}_D} k_{v_i} }{R_{D}}.
    \end{equation}
    
    \end{itemize}

\subsubsection{Overall Model Training Delay}
Taking all the computation and transmission delay components into account, the model training delay is given by
\begin{equation}\label{equ:model training delay}
	\begin{split}
		T(c) =  N_{loc} \left( T_{D,C} + T_{D,S} +  T_{D,B} +  T_{S,C} + T_{S,G} \right) \\
        + T_{D,U}
         + T_{S,D}.
	\end{split}
\end{equation}

Our goal is to find the optimal model partition that minimizes the overall training delay, i.e., ${\text{min}}~T(c)$. However, solving this optimization problem is challenging, especially when the structure of an AI model is complex and non-linear. 
{Linear AI models typically follow a sequential structure in which layers are executed one after another, e.g., LeNet and AlexNet. For linear AI models, the number of available cuts is limited, and the optimal cut can be efficiently obtained through a simple brute-force search. In contrast, non-linear AI models contain complex computational dependencies such as residual connections or parallel branches, where multiple paths may exist between layers (e.g., residual blocks in ResNet). These complex dependencies make it difficult to directly determine the optimal cut due to the large number of possible cuts and intricate data dependencies. 
}

\section{DAG-Based Representation of AI Models} \label{sec: DAG-Based AI Model Depiction}
In this section, we represent an AI model as a DAG, with edge weights encoding the computation and transmission delays. Based on this DAG, we formulate the optimal model partition problem as a minimum \textit{s-t} cut problem.

\subsection{DAG Model} \label{subsec: DAG Design}
{To find the optimal model partition, a unified and structured representation of AI models is required. Since an AI model inherently consists of layers with computational and data dependencies, it can naturally be abstracted as a DAG, where vertices represent layers and edges capture their dependencies. The acyclic property ensures that the execution order strictly follows the intrinsic computational logic of the AI model. Moreover, this representation enables us to explicitly encode computation and communication delays and facilitates the use of efficient graph-based optimization techniques, such as minimum \textit{s-t} cut algorithms, for optimal model partitioning. The construction of the DAG is described as follows.}

\textit{1) DAG Construction:}
As illustrated in Fig.~\ref{fig: AI model and DAG model}(a), a typical AI model consists of various types of layers, such as convolutional layers, pooling layers, and fully connected layers. These layers are connected via data dependencies, forming a computational graph. This graph can be represented as a DAG $\mathcal{G}_{A} = (\mathcal{V}_A, \mathcal{E}_A)$, where $\mathcal{V}_A = \{v_1, v_2, \dots, v_L\}$ denotes the set of layer vertices. {Each vertex $v_i$ corresponds to a specific layer and $L$ is the total number of layers in the AI model.} Let $\mathcal{E}_A$ be the set of directed edges, and each edge $(v_i, v_j) \in \mathcal{E}_A$ indicates that the output of layer $v_i$ is used as the input to layer $v_j$ in the AI model. Therefore, the AI model structure can be represented as a DAG.

To represent the computation flow from device to server, we introduce two virtual vertices: $v_D$, \textit{source vertex}, denoting a virtual device, and $v_S$, \textit{sink vertex}, denoting a virtual server. In addition, we add edges from the source vertex $v_D$ to each layer vertex $v_i \in \mathcal{V}$, and from each layer vertex $v_i \in \mathcal{V}$ to the sink vertex $v_S$, so that the entire AI model is enclosed between the source and sink vertices, as shown in Fig.~\ref{fig: AI model and DAG model}(b). The DAG can be defined as
\begin{equation}  \label{eq: DAG}
    \mathcal{G} = (\mathcal{V}, \mathcal{E}),
\end{equation}
{where $\mathcal{V} = \mathcal{V}_A \cup \{v_D, v_S\}$ and $\mathcal{E} = \mathcal{E}_A \cup \{ (v_D, v_i )\}_{v_i \in \mathcal{V}_A} \cup \{ (v_i, v_S)\}_{v_i \in \mathcal{V}_A}$. Note that $\{ (v_D, v_i )\}_{v_i \in \mathcal{V}_A}$ denotes the set of all directed edges from $v_D$ to each vertex $v_i \in \mathcal{V}_A$, and similarly, $\{ (v_i, v_S)\}_{v_i \in \mathcal{V}_A}$ denotes the set of all directed edges from $v_i \in \mathcal{V}_A$ to the sink $v_S$.}

\begin{figure}[t] 
    \renewcommand{\figurename}{Fig.}
    \centering
    \includegraphics[width=0.48\textwidth]{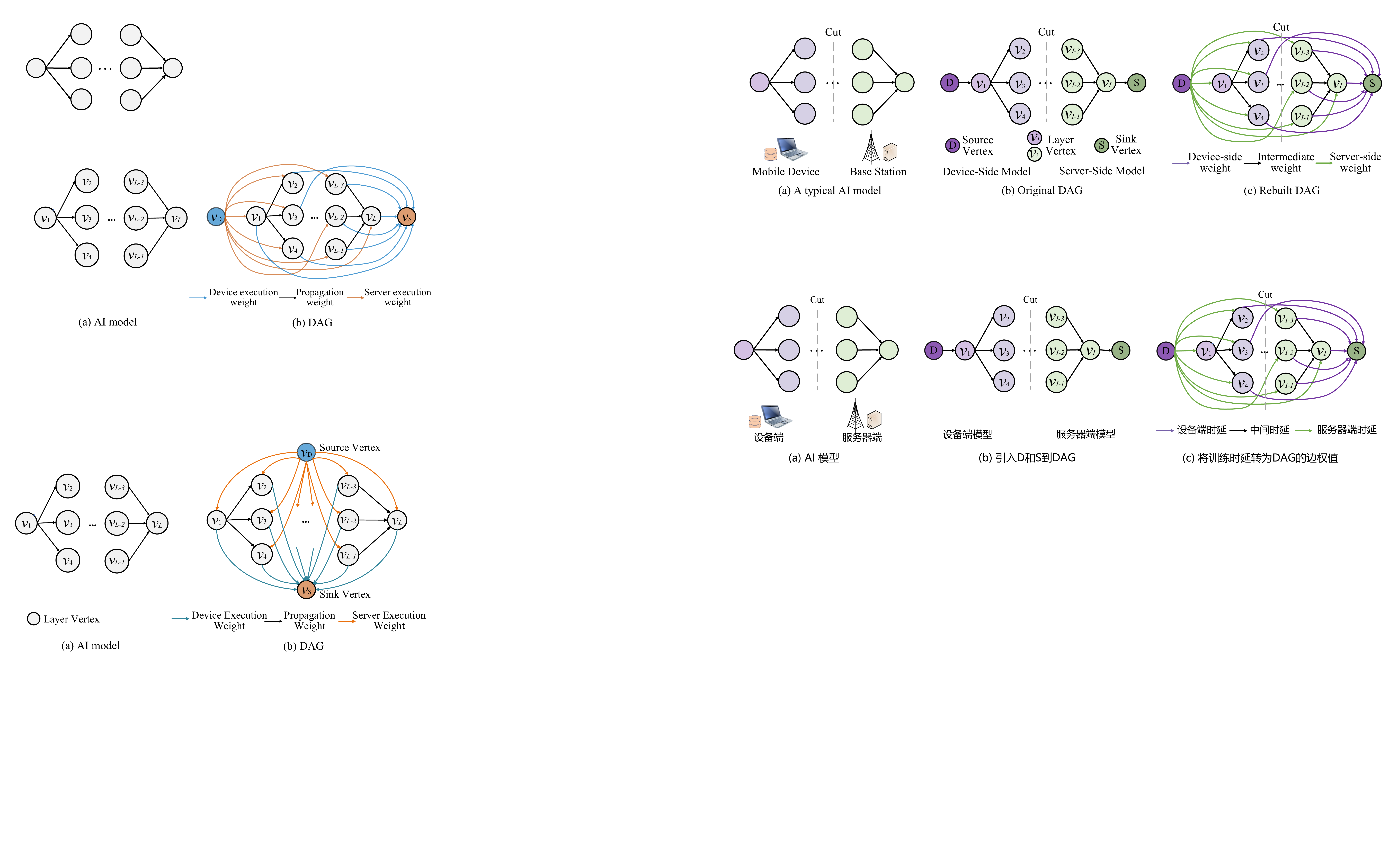}
    \caption{{Illustration of representing an AI model as a DAG. (a) The AI model consists of multiple layers. (b) The corresponding DAG adds a source and a sink vertex.}}
    \label{fig: AI model and DAG model}
\end{figure}

\begin{algorithm}[t] 
    \caption{DAG Building Process.}
    \label{ag: DAG Building Process}
    
    {\KwIn{$\mathcal{G}_{A} = (\mathcal{V}_A, \mathcal{E}_A)$\;}}
    {\KwOut{$\mathcal{G}$\;}}
    
    $\mathcal{V} = \mathcal{V}_A$, $\mathcal{E} = \mathcal{E}_A$\;
    \For{each $v_i \in \mathcal{V}$}{
        Add edges $(v_D, v_i)$ and $(v_i, v_S)$ to $\mathcal{E}$\;
    }

    \For{each $(v_i, v_j) \in \mathcal{E})$}{
        \uIf{$v_i == v_D$}{
            Calculate $w_{(v_i, v_j)}$ based on Eq.~(\ref{equ: device execution weight})\;
        }
        \uElseIf{$v_j == v_S$}{
            Calculate $w_{(v_i, v_j)}$ based on Eq.~(\ref{equ: server execution weight})\;
        }
        \Else{
            Calculate $w_{(v_i, v_j)}$ based on Eq.~(\ref{equ: propagation weight})\;
        }
        Assign $w_{(v_i, v_j)}$ to edge $(v_i, v_j)$\;
    }
    $\mathcal{G} = (\mathcal{V} \cup \{v_D, v_S\}, \mathcal{E})$.
\end{algorithm}
\textit{2) Mapping Delay into DAG:}
The main challenge in using DAG to solve the optimal model partitioning problem is that it cannot directly represent computation delay, device-side model uploading delay, or device-side model downloading delay via vertex or edge weights. This is because the computation delay of each layer depends on whether it is executed on the device or the server. In contrast, the uploading and downloading delays depend on the specific subset of layers trained on the device. To address this, the model training delays (including the computation delay and transmission delay) are mapped into three types of edge weights in the DAG as follows:

\begin{itemize}
    \item \textit{Device execution weight}: {The weight of the directed edge from each layer vertex $v_i \in \mathcal{V}$ to the sink vertex $v_S$ represents the device execution weight when $v_i$ is executed on the device. This weight is defined as the sum of the computation delay for training $v_i$ on the device and the transmission delay for uploading $v_i$ during the device-side model update.} It is calculated as
    \begin{equation}\label{equ: device execution weight}
       {w_{(v_i, v_S)} = N_{loc} \xi_{D, v_i} + \frac{k_{v_i}}{R_D}.}
    \end{equation}

    \item \textit{Server execution weight}: {The weight of the directed edge from source vertex $v_D$ to each layer vertex $v_i \in \mathcal{V}$ represents the server execution weight when $v_i$ is executed on the server. This weight is defined as the sum of the computation delay for training $v_i$ on the server and the transmission delay for downloading $v_i$ during the device-side model distribution.} It is calculated as
    \begin{equation}\label{equ: server execution weight}
        {w_{(v_D,v_i)}  = N_{loc} \xi_{S, v_i} + \frac{k_{v_i}}{R_S}.}
    \end{equation}

    \item \textit{Propagation weight:} The weight of the directed edge between any two layer vertices $v_i, v_j \in \mathcal{V}$ represents the propagation weight. This weight is defined as the sum of the transmission delay of $v_i$’s smashed data during forward propagation and $v_j$’s gradient during backward propagation. It is calculated as
    \begin{equation}\label{equ: propagation weight}
     \begin{split}
       w_{(v_i,v_j)} = N_{loc} \left(\frac{a_{v_i}}{R_D} + \frac{\tilde{a}_{v_i}}{R_S}\right).
     \end{split}
    \end{equation}
    Here, $\tilde{a}_{v_i}$ denotes the size of the gradient received by vertex $v_i$ from $v_j$ during backward propagation, which equals the size of the smashed data in forward propagation, i.e., $\tilde{a}_{v_i} = a_{v_i}$.
\end{itemize}

Through the above process, the computation and transmission delays are encoded as the weights of the edges in DAG $\mathcal{G}$. As shown in Fig.~\ref{fig: AI model and DAG model}(b), the edges with device execution, server execution, and propagation weights are colored blue, orange, and black, respectively.
The detailed procedure for constructing the DAG from an AI model is presented in Alg.~\ref{ag: DAG Building Process}. {Note that in the DAG, vertices are unweighted and represent only the layers of the AI model, as well as the virtual device and the virtual server.}

\subsection{Problem Formulation}
We aim to minimize the overall training delay in SL, which can be reformulated as a minimum \textit{s-t} cut problem in the DAG, i.e., finding an \textit{s-t} cut that minimizes the total weight of the edges intersected by the cut.  An \textit{s-t} cut partitions the vertex set into two disjoint subsets $\mathcal{V}_D$ and $\mathcal{V}_S$, such that the source vertex $v_D \in \mathcal{V}_D$ and the sink vertex $v_S \in \mathcal{V}_S$. The value of an \textit{s-t} cut ${c} = \{\mathcal{V}_D, \mathcal{V}_S\}$ is defined as the sum of the weights of all edges directed from  $\mathcal{V}_D$ to $\mathcal{V}_S$. Let $\mathcal{W}({c})$ denote the set of weights of these edges, then the value of the cut can be expressed as $T_{\mathcal{G}}({c}) = \sum_{w \in \mathcal{W}({c})} w$. Thus, the minimum \textit{s-t} cut problem in the DAG can be formulated as follows:
\begin{equation}	\label{eq: optimal problem in DAG}
   \begin{split}
    \underset{{c}}{\text{min}}~ &  T_{\mathcal{G}}({c})  \\
     \text{s.t.} ~~& c \in \mathcal{C},   \\
      ~~& v_D \in \mathcal{V}_D,   \\
     ~~& v_S \in \mathcal{V}_S,    \\ 
     ~~& v_i \in \mathcal{V}_S \Rightarrow v_j \in \mathcal{V}_S, ~ \forall (v_i,v_j) \in \mathcal{E}.
    \end{split}
\end{equation}
Here,  $\mathcal{C}$ is a set of all available cuts in the DAG, and the first constraint guarantees the feasibility of the cut selection. The second and third constraints ensure that the source vertex is placed on the device, while the sink vertex is placed on the server, respectively. The last constraint ensures that no vertex in $\mathcal{V}_S$ is the parent of any vertex in $\mathcal{V}_D$. This means that it is not possible to obtain an optimal cut, as device-side computations must wait for server-side computations to finish during forward propagation. 

{This problem cannot be directly solved using standard minimum \textit{s-t} cut methods. In the DAG, if the cut intersects multiple outgoing edges from the same parent vertex, the corresponding propagation weight may be counted multiple times, leading to an overestimation of the cut value. However, in the AI model, the smashed data and its corresponding gradient from a parent vertex only need to be transmitted once, regardless of how many outgoing edges it has.} To address this issue, we introduce a general model partitioning algorithm in the following section.


\section{General Model Partitioning Algorithm} \label{sec: DAG-based model partition Algorithm}
\begin{figure}[t] 
	\renewcommand{\figurename}{Fig.}
	\centering
	\includegraphics[width=0.35\textwidth]{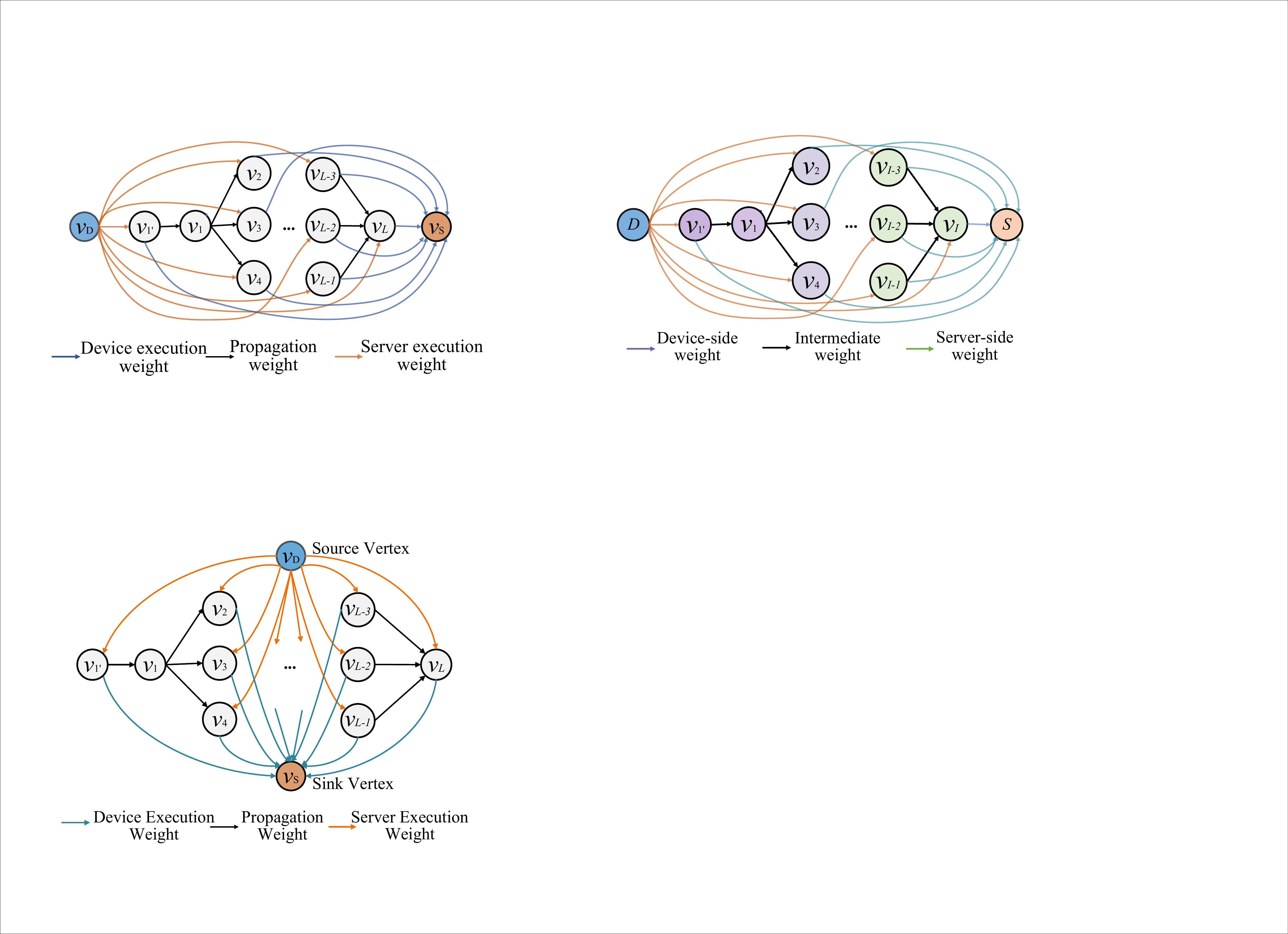}
	\caption{{Illustration of adding auxiliary vertices (e.g., $v_{1'}$) for parent vertices that have multiple child vertices.}}
	\label{fig: adding the auxiliary vertex}
\end{figure}
In this section, we propose a general model partitioning algorithm to determine the optimal model partition for arbitrary AI model architectures. We also prove that the cut obtained by this algorithm is optimal.

\subsection{Algorithm Design}
{To eliminate the overestimation of propagation weights in the DAG, we propose a general model partitioning algorithm.} This algorithm restructures the DAG to prevent duplicate transmission-weight counts from the same parent vertex. The main steps are described as follows.

\begin{enumerate}
\item Initialize a DAG $\mathcal{G}' = (\mathcal{V}' \cup \{v_D, v_S\}, \mathcal{E}')$ by setting $\mathcal{V}' = \mathcal{V}$ and $\mathcal{E}' = \mathcal{E}$. Then, identify the set of parent vertices that have multiple child vertices in $\mathcal{G}'$, denoted as $\mathcal{V_P}$.

\item For each parent vertex $v_p \in \mathcal{V_P}$, an auxiliary vertex $v_{p'}$ is introduced into the DAG, as illustrated in Fig.~\ref{fig: adding the auxiliary vertex}.

\item All incoming edges originally connected to $v_p$ are redirected to the new vertex $v_{p'}$, while preserving their original weights:
\begin{equation} \label{equ: the edges aggregated to the parent vertex}
    w_{(v_{i}, v_{p'})} = w_{(v_i,v_p)},
\end{equation}
{where $v_i$ is a parent vertex of $v_p$. For example, in Fig.~\ref{fig: adding the auxiliary vertex}, edge $(v_D, v_1)$ becomes $(v_D, v_{1'})$, satisfying $w_{(v_D, v_{1'})} = w_{(v_D, v_1)}$.}

\item The edge from $v_p$ to the sink vertex $v_S$ is also redirected to originate from $v_{p'}$, with its weight unchanged:
\begin{equation}\label{equ: auxiliary vertex weight with other vertex}
    w_{(v_{p'}, v_S)} = w_{(v_p, v_S)}.
\end{equation}
As shown in Fig.~\ref{fig: adding the auxiliary vertex}, edge $(v_1, v_S)$ becomes $(v_{1'}, v_S)$, and we have $w_{(v_{1'}, v_S)} = w_{(v_1, v_S)}$.

\item Add a new edge from the auxiliary vertex to the original parent vertex, i.e., $(v_{p'}, v_p)$, with a weight equal to the propagation cost from $v_p$ to one of its child vertices:
\begin{equation}\label{equ: auxiliary vertex weight with source vertex}
    w_{(v_{p'},v_{p})} = w_{(v_p,v_j)},
\end{equation}
where $v_j$ is a child of $v_p$. For example, in Fig.~\ref{fig: adding the auxiliary vertex}, edge $(v_{1'}, v_1)$ is added with $w_{(v_{1'},v_1)} = w_{(v_1,v_2)}$.
\end{enumerate}

After the above process, we build a new DAG $\mathcal{G}'$ for non-linear neural networks. This new DAG resolves the issue of counting the same propagation weight multiple times when a cut intersects multiple edges from the same parent vertex. The optimal cut can then be obtained by solving the minimum \textit{s-t} cut problem on $\mathcal{G}'$, which can be efficiently solved using standard maximum-flow algorithms, such as Dinic’s algorithm~\cite{Dinitz2006}. For linear neural networks, this issue does not exist, as each layer has at most one child. Thus, the optimal cut can be directly computed on the original DAG $\mathcal{G}$ using simple algorithms such as the brute-force search method. The detailed procedure of the general model partitioning algorithm is presented in Alg.~\ref{ag: general model partitioning algorithm}.

\begin{algorithm}[t] 
    \caption{General Model Partition Algorithm.}
    \label{ag: general model partitioning algorithm}
    {\KwIn{$\mathcal{G} = (\mathcal{V} \cup \{v_D, v_S\}, \mathcal{E})$\;}}
    {\KwOut{$c^{\ast}$\;}}
    Find out the set $\mathcal{V_P}$ of parent vertices with several child vertices\;
    \uIf{$\mathcal{V_P} = \emptyset$}{
    Obtain minimum \textit{s-t} cut $c_{\mathcal{G}}^{\ast}$ of $\mathcal{G}$ via the brute-force search\;
    $c^{\ast} = c_{\mathcal{G}}^{\ast}$\;
    }
    \Else{
    $\mathcal{E}' \gets \mathcal{E}$, $\mathcal{V}' \gets \mathcal{V}$\;
    \For{$v_p$ in $\mathcal{V_P}$}{
        Add auxiliary vertex $v_{p'}$ to $\mathcal{V}'$\;
        The edges pointing to parent vertex $v_p$ are migrated from $v_p$ to $v_{p'}$\;
        Calculate $w_{(v_{i}, v_{p'})}$ based on Eq.~(\ref{equ: the edges aggregated to the parent vertex})\;
        Assign $w_{(v_{i}, v_{p'})}$ to edge $(v_{i}, v_{p'})$ in $\mathcal{G}'$\;
        The edge between $v_p$ and $v_S$ are migrated from $v_p$ to $v_{p'}$\;
        Calculate $w_{(v_{p'}, v_S)}$ based on Eq.~(\ref{equ: auxiliary vertex weight with other vertex})\;
        Assign $w_{(v_{p'}, v_S)}$ to edge $(v_{p'}, v_S)$ in $\mathcal{G}'$\;
        Add an edge $(v_{p'}, v_p)$ to $\mathcal{E}'$\;
        Calculate $w_{(v_{p'},v_{p}})$ based on Eq.~(\ref{equ: auxiliary vertex weight with source vertex})\;
        Assign $w_{(v_{p'},v_{p}})$ to edge $(v_{p'},v_{p})$ in $\mathcal{G}'$\;
    }
    $\mathcal{G}' = (\mathcal{V}' \cup \{v_D, v_S\}, \mathcal{E}')$\;
    Obtain minimum \textit{s-t} cut $c_{\mathcal{G}'}^{\ast}$ of $\mathcal{G}'$ via the maximum flow method\;
    $c^{\ast} = c_{\mathcal{G}'}^{\ast}$.}
\end{algorithm}

\subsection{Performance Analysis}
{For linear neural networks, it is straightforward to see that the minimum \textit{s-t} cut in the DAG is precisely equivalent to the optimal model partition in the AI model. Therefore, in this subsection, we focus on proving the equivalence between the optimal model partition in non-linear neural networks and the minimum \textit{s-t} cut in the DAG $\mathcal{G}'$.}
To prove this equivalence, it begins with the following reasonable assumption.
\begin{assumption} \label{assumption 1}
    The server's computing capability is at least as strong as the device's. Specifically, for any layer, the computation latency on the server is no greater than that on the device, i.e.,
     \begin{equation}\label{equ: computing capability of assumption}
  \begin{split}
    \xi_{D, v_i} - \xi_{S, v_i} \ge 0, ~\forall v_i \in \mathcal{V}.
  \end{split}
 \end{equation}
\end{assumption}
Under Assumption~\ref{assumption 1}, we can rigorously prove that the optimal model partition in the AI model is equivalent to finding the minimum \textit{s-t} cut in the corresponding DAG. This relationship is formalized in the following theorem.
\begin{theorem} {The minimum \textit{s-t} cut in the DAG $\mathcal{G}'$ corresponds exactly to the optimal model partition in the AI model.}
\label{theorem: min-cut is equivalent to min delay}
\end{theorem}

\begin{proof}
{Please refer to Appendix~A.}
\end{proof}

\section{Block-Wise Model Partitioning Algorithm} \label{sec: block-wise model partition Algorithm}

In this section, we present a block-wise model partitioning algorithm to determine the optimal model partition for the AI models with block structures. 

\subsection{Intra-Block Cut Analysis} \label{sec: Intra-Block Cut Analysis}
{In an AI model, a block is a set of layers that is often reused multiple times. When an AI model contains block structures, the computational complexity of the optimal model partition can be further reduced.} For example, the inception block is reused 9 times in GoogLeNet~\cite{szegedy2015going}, the residual block is reused 8 and 16 times in ResNet18 and ResNet50~\cite{he2016deep}, respectively, and the dense block is reused 58 and 98 times in DenseNet121 and DenseNet201~\cite{huang2017densely}, respectively.
These blocks share the same internal architecture and transformation logic, although their input and output dimensions may vary depending on their position in the model. Because of this structural repetition, analyzing a block allows us to generalize its behavior across the other blocks.

\textit{1) Block Detection:}  
To leverage block-level regularity, we first need to identify repeated blocks in the AI model. The detection process is based on the observation that many blocks exhibit a branching-aggregation structure, i.e., multiple parallel paths diverge from a single parent vertex and subsequently converge into a single vertex. When a parent vertex has multiple child vertices, these child vertices are added into a candidate set $\mathcal{V}_{B}^{m}$. We then continue tracking the successors of these child vertices until they all converge to the same vertex. At this stage, all intermediate successors and the final converged vertex are added into $\mathcal{V}_{B}^{m}$. The internal edges of this block are denoted as $\mathcal{E}_{B}^{m}$. Consequently, we obtain a block-level DAG denoted as $\mathcal{G}_{B} = (\mathcal{V}_{B}^{m}, \mathcal{E}_{B}^{m})$. We continue this detection process towards the output end of the AI model, and if $\mathcal{G}_{B}$ appears multiple times, it is retained as a reusable unit. The detailed block-detection procedure is shown in Alg.~\ref{alg:block-detection}.

\textit{2) Intra-Block Cut Identification:} {The intra-block cut identification is designed to determine whether the optimal model partition intersects the internal edges of any block. The identification only relies on the sizes of smashed data within the block and does not require device or network parameters. If the optimal model partition intersects a block's internal edges, the block must be preserved in full. Otherwise, it can be safely abstracted as a single vertex of the DAG, simplifying the DAG structure and accelerating the cut selection process.}

\begin{algorithm}[t]
\caption{Block-Detection Procedure.}
\label{alg:block-detection}
{\KwIn{$\mathcal{G}_A = (\mathcal{V}_A, \mathcal{E}_A)$\;}
\KwOut{Set of blocks $\mathbf{V}_{B} = [\mathcal{V}_{B}^{1}, \dots, \mathcal{V}_{B}^{M}]$\;}
Initialize $\mathbf{V}_{B} = \emptyset$, $M=0$\;
\For{each $v_i \in \mathcal{V}_A $}{
    \If{$v_i$ has multiple children}{
        $\mathcal{V}_{B}^{m} = \{\text{children of } v_i \}$\;
        Set $\text{queue} \gets \text{children of } v_i$\;
        \While{successors of vertices in queue do not converge}{
            Add successors of queue into $\mathcal{V}_{B}^{m}$\;
            Update $\text{queue} \gets$ successors of queue\;
        }
        Add the converged vertex into $\mathcal{V}_{B}^{m}$\;
        Extract $\mathcal{E}_{B}^{m}$ as internal edges among $\mathcal{V}_{B}^{m}$\;
    }
    $M \gets M+1$, $\mathbf{V}_{B} \gets \mathbf{V}_{B} \cup \{\mathcal{V}_{B}^{m}\}$.
}
}
\end{algorithm}

A block DAG is denoted by $\mathcal{G}_{B} = (\mathcal{V}_{B},~\mathcal{E}_{B})$, where each vertex in $\mathcal{V}_{B}$ denotes a layer within the block and each directed edge in $\mathcal{E}_{B}$ indicates the flow of smashed data between layers. The input layer of the block is denoted as vertex $v_{\text{in}}$. As shown in Fig.~\ref{fig: The DAG of a block}, $\mathcal{V}_{B} = \{v_1, \dots, v_7\}$ and $v_{\text{in}}$ is the input layer. The vertices are unweighted, and each edge is weighted by the size of the smashed data from its parent vertex. The input and output layers of the block are treated as the source and sink of $\mathcal{G}_{B}$, respectively.
{Let $c^{\text{in}}_B$ be the cut that separates the input layer from the block, and $a^{\text{in}}_B$ be the size of smashed data that needs to be transmitted under this cut. Specifically, $a^{\text{in}}_B$ is equal to the size of the smashed data generated at the input vertex $v_{\text{in}}$. Let $c^{\min}_B$ denote the minimum \textit{s-t} cut of the block, and $a^{\min}_B$ be the size of smashed data that needs to be transmitted under this cut. $a^{\min}_B$ can be computed using standard maximum-flow algorithms~\cite{stoer1994simple, Dinitz2006}.}

{If $a_B^{\min} \geq a^{\text{in}}_B$, the optimal cut does not intersect any intra-block edges. Otherwise, it may intersect the intra-block edges.}
Under Assumption~\ref{assumption 1}, we have the following theorem.
\begin{theorem} {If $a_B^{\min} \geq a^{\text{in}}_B$, the optimal cut does not intersect any intra-block edges.}
\label{theorem: not intersect}
\end{theorem}

\begin{proof}
    Please refer to Appendix~B.
\end{proof}

\begin{figure}[t] 
	\renewcommand{\figurename}{Fig.}
	\centering
	\includegraphics[width=0.23\textwidth]{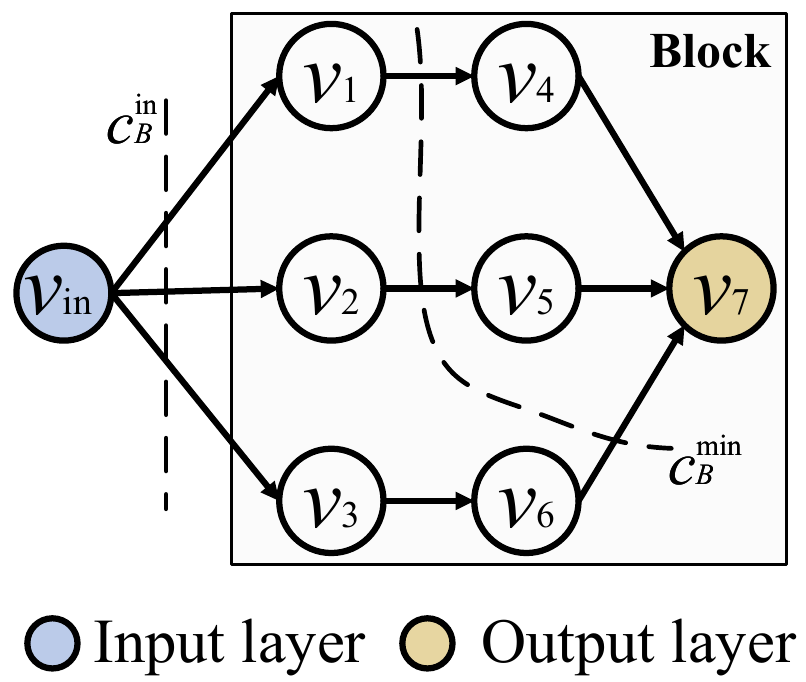}
	\caption{An example of a block DAG, where each edge weight indicates the size of smashed data.}
	\label{fig: The DAG of a block}
\end{figure}

\subsection{Block-Level Abstraction} \label{subsec: Block-Level Abstraction}

{When the optimal model partition does not intersect any intra-block edges, each repeated block in the AI model can be abstracted as a single vertex in the DAG. The detailed
descriptions of the block-level abstraction are as follows.}

\textit{ 1) Procedure:} {Firstly, let $\overline{\mathcal{G}} = (\overline{\mathcal{V}} \cup \{v_D, v_S\}, \overline{\mathcal{E}})$ denote a simplified DAG, where $\overline{\mathcal{V}} = \mathcal{V}$ and $\overline{\mathcal{E}} = \mathcal{E}$. {Let $M$ denote the number of blocks in $\overline{\mathcal{G}}$, and the set of blocks is indexed by $ \mathbf{V}_{B} = [\mathcal{V}_{B}^{1}, \mathcal{V}_{B}^{2}, \dots, \mathcal{V}_{B}^{M}]$.} Each block is sequentially replaced by a single vertex in $\overline{\mathcal{G}}$.} Let $v^m_B$ denote the vertex that replaces block $\mathcal{V}_{B}^{m}$. 

Secondly, the edge weights between $v^m_B$ and other vertices in $\overline{\mathcal{G}}$ are calculated based on the original connections of block $\mathcal{V}_{B}^{m}$. The weight of the edge from $v^m_B$ to sink vertex $v_S$, denoted by $w_{(v^m_B, v_S)}$, is the sum of the device execution weights of all vertices within block $\mathcal{V}_{B}^{m}$, i.e., 
\begin{equation}\label{equ: The weight between moudle and server}
 \begin{split}
   w_{(v^m_B, v_S)} = \sum_{v_b \in \mathcal{V}_{B}^{m}}  w_{(v_b, v_S)}.
 \end{split}
\end{equation}

Thirdly, the weight of the edge from source vertex $v_D$ to $v^m_B$ in $\overline{\mathcal{G}}$, denoted by $w_{(v_D, v^m_B)}$, is the sum of the server execution weights of all vertices within this block, i.e.,
\begin{equation}\label{equ: The weight between device and block}
 \begin{split}
   w_{(v_D, v^m_B)} = \sum_{v_b \in \mathcal{V}_{B}^{m}}  w_{(v_D, v_b)}.
 \end{split}
\end{equation}

Fourthly, for any parent vertex of block $\mathcal{V}_{B}^{m}$, denoted as $v_{p}$, the weight of the edge from $v_p$ to  vertex $v^m_B$, denoted by $w_{(v_p, v^m_B)}$, is given by
\begin{equation}\label{equ: The weight between front vertex and block}
 \begin{split}
   w_{(v_p, v^m_B)} = w_{(v_p, v_b)},
 \end{split}
\end{equation}
where $v_b$ is any vertex in block $\mathcal{V}_{B}^{m}$. 

Fifthly, for any child vertex of block $\mathcal{V}_{B}^{m}$, denoted as $v_{\tilde{b}}$, the weight of the edge from vertex $v_{B}^{m}$ to vertex $v_{\tilde{b}}$, denoted by $w_{(v_{B}^{m}, v_{\tilde{b}})}$, is given by
\begin{equation}\label{equ: The weight between behind vertex and block}
 \begin{split}
   w_{(v_{B}^{m}, v_{\tilde{b}})} = \sum_{v_b \in \mathcal{V}_{B}^{m}} w_{(v_b, v_{\tilde{b}})}.
 \end{split}
\end{equation}
The same abstraction procedure is applied to all blocks in the model. 
\begin{figure}[t] 
	\renewcommand{\figurename}{Fig.}
	\centering
	\includegraphics[width=0.45\textwidth]{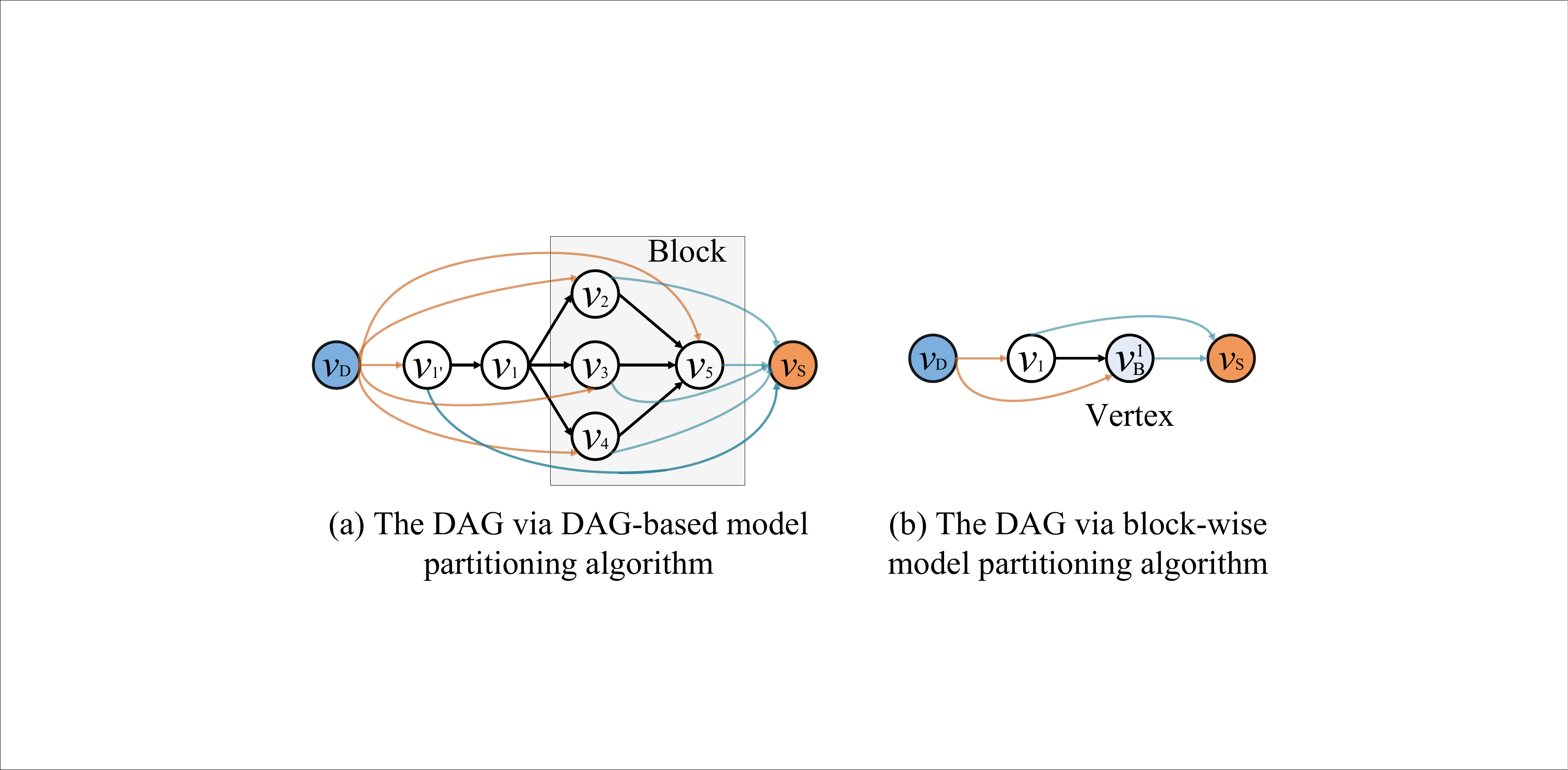}
	\caption{An example of DAG transformation via general and block-wise model partitioning algorithms.}
	\label{fig: simplified DAG}
\end{figure}


\textit{ 2) Example:} To illustrate the above transformation process more clearly, we provide an example as follows. Fig.~\ref{fig: simplified DAG}(a) shows a DAG $\mathcal{G}'$ constructed by the general model partitioning algorithm, where a block is highlighted. Its corresponding DAG $\overline{\mathcal{G}}$ constructed by the block-wise model partitioning algorithm, is shown in Fig.~\ref{fig: simplified DAG}(b). In $\overline{\mathcal{G}}$, the entire block is replaced by a single vertex $v^1_B$. The weight of the edge from this integrated vertex $v^1_B$ to the source vertex $v_S$, denoted as $w_{(v^1_B, v_S)}$, is the sum of the weights of all original edges from $v_S$ to the vertices within the block:
\begin{equation}\label{equ: server execution weight in simplified DAG}
 \begin{split}
   w_{(v^1_B, v_S)} = w_{(v_2, v_S)} + w_{(v_3, v_S)} + w_{(v_4, v_S)} + w_{(v_5, v_S)}.
 \end{split}
\end{equation}
Similarly, the weight of the edge from the sink vertex $v_D$ to $v^1_B$, denoted as $w_{(v_D, v^1_B)}$, is the sum of the weights of all original edges from the intra-block vertices to $v_D$:
\begin{equation}\label{equ: device execution weight in simplified DAG}
 \begin{split}
   {w_{(v_D, v^1_B)} = w_{(v_D, v_2)} + w_{(v_D, v_3)} + w_{(v_D, v_4)} + w_{(v_D, v_5)}.}
 \end{split}
\end{equation}
The weight of the edge from vertex $v_1$ to integrated vertex $v^1_B$, denoted as $w_{(v_1, v^1_B)}$, is the weight of any one of the edges from $v_1$ to the intra-block vertices:
\begin{equation}\label{equ: propagation weight in simplified DAG}
 \begin{split}
   w_{(v_1, v^1_B)} = w_{(v_1, v_2)}.
 \end{split}
\end{equation}

For the DAG in Fig.~\ref{fig: simplified DAG}(a), there are 17 edges and eight vertices. After the block integration by using the block-wise model partitioning algorithm, there are only five edges and four vertices in the simplified DAG, as shown in Fig.~\ref{fig: simplified DAG}(b). The block-wise model partitioning algorithm can make the DAG more streamlined, thereby reducing the computational complexity of finding the optimal model partition.

\subsection{Algorithm Execution}
{For an arbitrary AI model, we first run the block-detection procedure to check whether the model contains repeated blocks. If no repeated blocks are found, we directly apply the general model-partitioning algorithm. Otherwise, we perform intra-block cut identification to determine whether cutting within a block can be optimal (based on the condition derived in Section~\ref{sec: Intra-Block Cut Analysis}). If so, we keep the blocks expanded in the DAG. Otherwise, each repeated block is abstracted as a single vertex to simplify the DAG, and the optimal model partition can then be obtained more efficiently. The detailed procedure of the proposed block-wise model-partitioning algorithm is presented in Alg.~\ref{ag: Block-Wise Model Partition Algorithm}.}

\begin{algorithm}[t]
    \caption{Block-Wise Model Partition Algorithm.}
    \label{ag: Block-Wise Model Partition Algorithm}
    {\KwIn{$\mathcal{G} = (\mathcal{V} \cup \{v_D, v_S\}, \mathcal{E})$\;}}
    {\KwOut{$c^{\ast}$\;}}
    Calculate $a^{\text{in}}_B$ and $a_B^{\min}$ based on the intra-block cut detection\;
    \uIf{$a_B^{\min} < a^{\text{in}}_B $}{
    \tcp{The optimal cut intersects intra-block edges.}
    Obtain optimal cut $c^{\ast}$ by inputting DAG $\mathcal{G}$ into Algorithm~\ref{ag: general model partitioning algorithm}\; }
    \Else{
    \tcp{The optimal cut does not intersect any intra-block edges.}
    $\overline{\mathcal{E}} \gets \mathcal{E}$, $\overline{\mathcal{V}} \gets \mathcal{V}$\;
    Identify and define the block set $\mathbf{V}_{B}$\;
    \For{$\mathcal{V}_{B}^{m}$ in $\mathbf{V}_{B}$}{
        Replace block $\mathcal{V}_{B}^{m}$ with vertex $v_{B}^{m}$ in $\overline{\mathcal{V}}$\; 
        Add edges $(v_n, v^m_B)$, $(v_D, v^m_B)$, $(v^m_B, v_{\tilde{n}})$, and$ (v^m_B, v_S)$ into $\overline{\mathcal{E}}$\;
        Calculate these edge weights connected to vertex $v_{B}^{m}$ based on Eqs.~(\ref{equ: The weight between moudle and server})–(\ref{equ: The weight between behind vertex and block})\;
        Assign edge weights to these edges\;

    }
    $\overline{\mathcal{G}} = (\overline{\mathcal{V}} \cup \{v_D, v_S\}, \overline{\mathcal{E}})$\;
    Obtain optimal cut $c^{\ast}$ by inputting DAG $\overline{\mathcal{G}}$ into Algorithm~\ref{ag: general model partitioning algorithm}.
    }
\end{algorithm}

\subsection{Computational Complexity Analysis} \label{sec: Computational Complexity Analysis}

{The computational complexity of the proposed model-partitioning algorithms depends on the structure of the AI model, i.e., whether it is linear or non-linear. For a linear neural network, the optimal model partition can be directly obtained using a simple brute-force search method on the AI model, since no parent vertex has multiple outgoing edges.} Therefore, the computational complexity of the proposed algorithm is $O(L)$ for a linear neural network, where $L$ is the number of layers. For a non-linear neural network, we represent it as a DAG and compute the minimum \textit{s-t} cut using a maximum flow algorithm. In this work, we adopt Dinic’s algorithm~\cite{Dinitz2006}, which is known for its efficiency in network flow problems. The computational complexity of Dinic's algorithm is $O(|\mathcal{V}|^{2}|\mathcal{E}|)$, where $|\mathcal{V}|$ is the number of vertices, and $|\mathcal{E}|$ is the number of edges in the DAG. Therefore, the computational complexity of our proposed algorithm is $O(|\mathcal{V}|^{2}|\mathcal{E}|)$ for non-linear neural networks.

If a brute-force search is applied to a non-linear neural network, the total number of candidate cut configurations is \( 2^{|\mathcal{V}|} \), where each subset denotes a possible cut. For each configuration, connectivity validation is performed, which incurs a complexity of \( O(|\mathcal{V}| + |\mathcal{E}|)\). Hence, the overall computational complexity of the brute-force method is \( O(2^{|\mathcal{V}|} \cdot (|\mathcal{V}| + |\mathcal{E}|)) \), making it impractical for large-scale non-linear neural networks. In contrast, our proposed algorithm achieves significantly lower computational complexity compared with the brute-force search method when applied to non-linear neural networks.

\subsection{{Discussion}}

The proposed model partitioning algorithms can be extended to minimize the overall system energy consumption during model training or inference. This can be achieved by redefining the edge weights in the constructed DAG to represent the corresponding energy consumption. In this way, the minimum \textit{s-t} cut naturally yields the energy-optimal model partition, enabling energy-aware split strategies for various training and inference scenarios. Furthermore, the proposed model partitioning algorithms can also be extended to transformer-based large language models (LLMs). In LLMs, components such as the embedding layer, transformer blocks, and classification head can be treated as distinct blocks. By applying the block-wise model partitioning algorithm, the optimal model partition of an LLM can be efficiently determined.

\section{Performance Evaluation}\label{sec: simulation results}
In this section, we first evaluate the performance of the proposed general and block-wise model partitioning algorithms for determining the optimal model partition. Then, we assess the performance of the proposed solution in wireless networks.

\subsection{Performance of the Proposed Algorithms}
\subsubsection{Simulation Setup}
In this subsection, the proposed general and block-wise model partitioning algorithms are evaluated by executing them on a server to obtain the optimal model partition. This server is equipped with an Intel Xeon(R) Gold 6226R CPU. The proposed partition algorithms are compared with the following benchmarks:  
\begin{itemize}
    \item \textbf{Brute-force}~\cite{lin2024efficient}, which exhaustively examines all possible subsets to find the optimal model partition.

    \item \textbf{Regression}~\cite{wen2025training}, which fits the quantitative relationship between the cut position and key factors in an AI model, thereby transforming the optimal model partitioning problem into a continuous optimization problem. The optimal model partition is then obtained by solving this optimization problem.
\end{itemize}
Since the regression method cannot handle non-linear neural networks directly, the block-level abstraction described in Section~\ref{subsec: Block-Level Abstraction} is applied to convert the model into a linear form, enabling compatibility with the regression approach.

To compare the cuts obtained by the proposed algorithms with those from the brute-force search method, we evaluate AI models containing non-linear blocks. However, due to the exponential computational complexity of brute-force search, applying it to non-linear models is impractical. {Therefore, we consider three representative block-based networks, each incorporating a single type of non-linear block, as illustrated in Fig.~\ref{fig: DNN block}. These blocks are:
\textit{residual blocks}~\cite{targ2016resnet}, which introduce skip connections that add inputs directly to the outputs of convolutional layers and are widely used in ResNet;
\textit{inception blocks}~\cite{szegedy2017inception}, which combine multiple convolutional filters of varying kernel sizes and are commonly adopted in GoogLeNet;
and \textit{dense blocks}~\cite{huang2018condensenet}, which connect each layer to all subsequent layers to enhance feature reuse and are widely used in DenseNet.}

In addition, to evaluate the of the proposed algorithms, experiments are conducted on four widely adopted AI models with varying depths and architectural patterns: \textit{GoogLeNet}, a 22-layer model with nine inception blocks, an average layer output size of 22 MB, and a total computational load of 50 GFlops; \textit{ResNet18}, an 18-layer model with eight residual blocks, an average layer output size of 1.1 MB, and a total computational load of 1.8 GFlops; \textit{ResNet50}, a 50-layer model with 16 residual blocks, an average layer output size of 1.5 MB, and a total computational load of 3.7 GFlops; and \textit{DenseNet121}, a 121-layer model with 58 dense blocks, an average layer output size of 29 MB, and a total computational load of 92 GFlops.

The \textit{running time} is the time taken by the server to execute an algorithm and identify the optimal model partition for SL. The \textit{probability of the optimal cut} refers to the probability that the cut obtained by an algorithm exactly matches the one obtained by the brute-force search method.

\begin{figure}[t]
    \centering
    \renewcommand{\thesubfigure}{(\alph{subfigure})}
    \begin{subfigure}[b]{0.34\textwidth}
        \includegraphics[width=\textwidth]{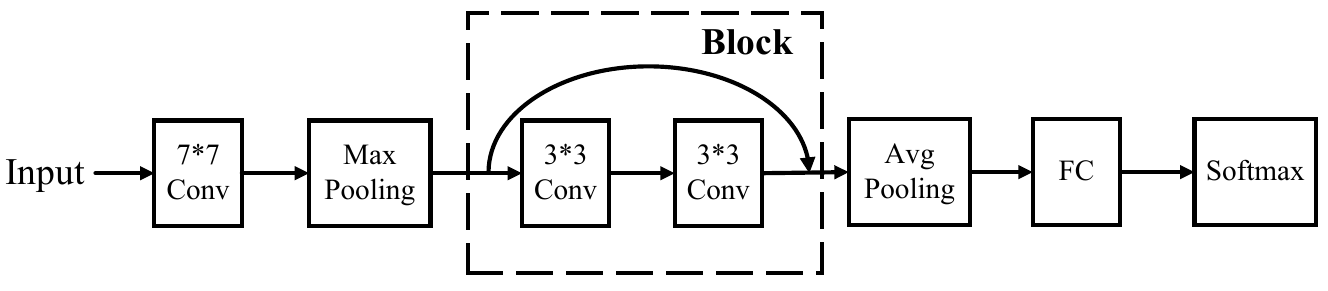} 
        \caption{Network with a residual block}
    \end{subfigure}
    \hfill
    \begin{subfigure}[b]{0.45\textwidth}
        \includegraphics[width=\textwidth]{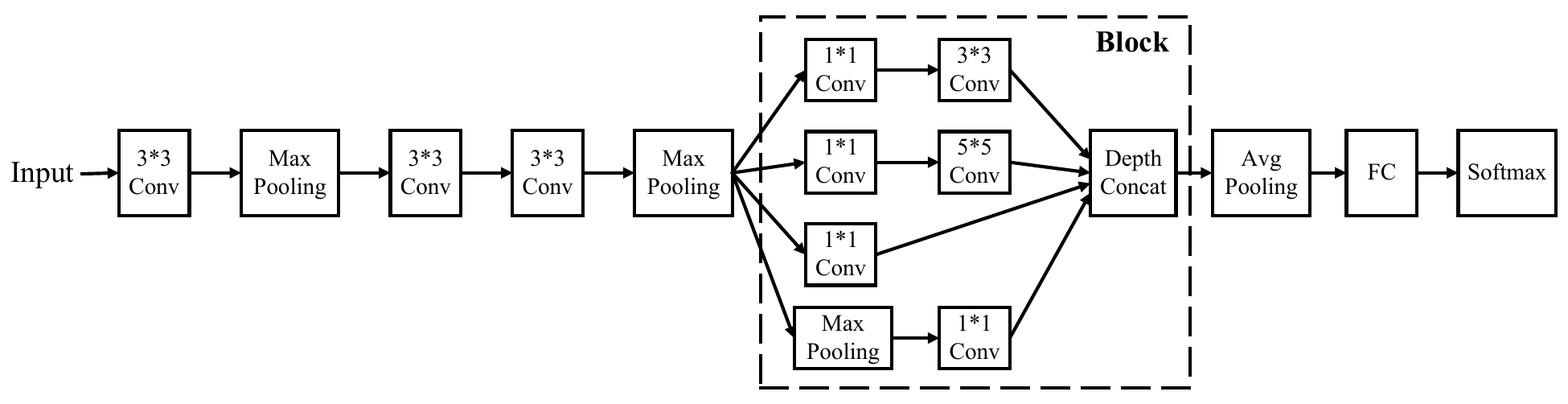} 
        \caption{Network with an inception block}
    \end{subfigure}
    \hfill
    \begin{subfigure}[b]{0.49\textwidth}
        \includegraphics[width=\textwidth]{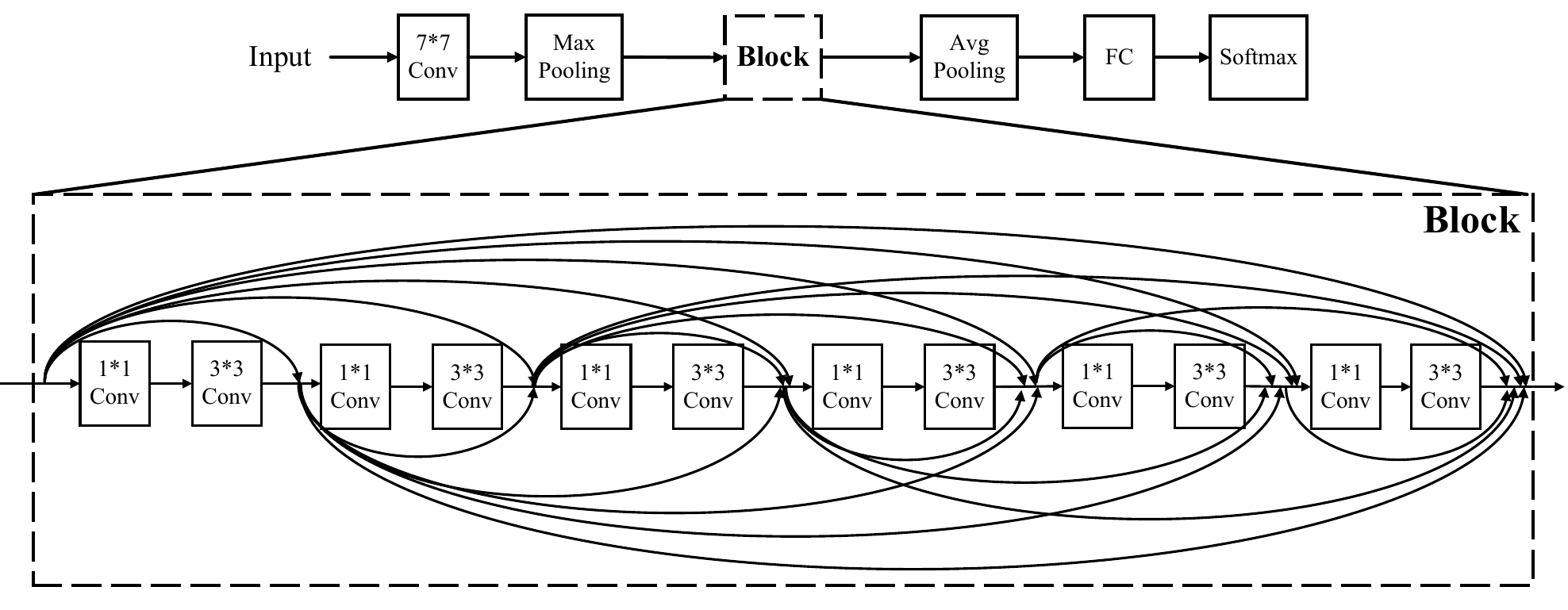} 
        \caption{Network with a dense block}
    \end{subfigure}
    
    \caption{{The structures of block-based networks with residual, inception, and dense blocks, which exhibit different intra-block connections.}}
    \label{fig: DNN block}
\end{figure}

\subsubsection{Effectiveness of the Proposed Algorithms} 

Since the regression method has a constant-time computational complexity of $\mathcal{O}(1)$, its theoretical complexity is omitted from the experimental analysis. Fig.~\ref{fig: complexity and propagation in DNN block}(a) illustrates the computational complexity comparison between our proposed algorithms and the brute-force search method on networks with a single block. The general model partitioning algorithm significantly outperforms the brute-force search method, achieving reductions in computational complexity by $1.9 \times$, $143.3 \times$, and $166.1 \times$ for networks with residual, inception, and dense blocks, respectively. Moreover, the block-wise model partitioning algorithm further streamlines the process, achieving additional reductions in complexity of $3.2 \times$, $4.9 \times$, and $66.9 \times$ over the general model partitioning algorithm. This is because, in block-structured AI models, the block-wise model partitioning algorithm simplifies the DAG architecture, enabling more efficient identification of the optimal model partition.

Figure~\ref{fig: complexity and propagation in DNN block}(b) presents the average probability of the optimal model partition over 1,000 simulation runs on networks with a single block. Although the regression method has the lowest computational complexity, it suffers from a low probability of the optimal model partition. For example, on the networks with residual and dense blocks, the probability is only 73.6\%. In particular, its probability drops to zero on the network with an inception block. This is because the regression method fails to accurately capture the relationship between the split point and model parameters, particularly the size of the smashed data. In contrast, our proposed algorithms obtain the optimal model partition across all three network types. This is because the cuts are optimal by the proposed algorithms, which is proved in Theorem~\ref{theorem: min-cut is equivalent to min delay}.  

Figure~\ref{fig: Computational complexity in AI models} presents the computational complexity of three algorithms on different AI models. The general model partitioning algorithm significantly reduces computational complexity compared to the brute-force search method. In contrast, the block-wise model partitioning algorithm further reduces it compared to the general model partitioning algorithm. Specifically, for DenseNet121, the general model partitioning algorithm achieves a $5.8 \times 10^{33}$ reduction in computational complexity compared to the brute-force search method. The block-wise model partitioning algorithm further reduces the complexity by an additional factor of $1.7 \times 10^3$ compared to the general model partitioning algorithm. That is because DenseNet121 has the most repeated blocks among these models.

\begin{figure}[t]
    \centering
    \renewcommand{\thesubfigure}{(\alph{subfigure})}
    \begin{subfigure}[b]{0.24\textwidth}
        \includegraphics[width=\textwidth]{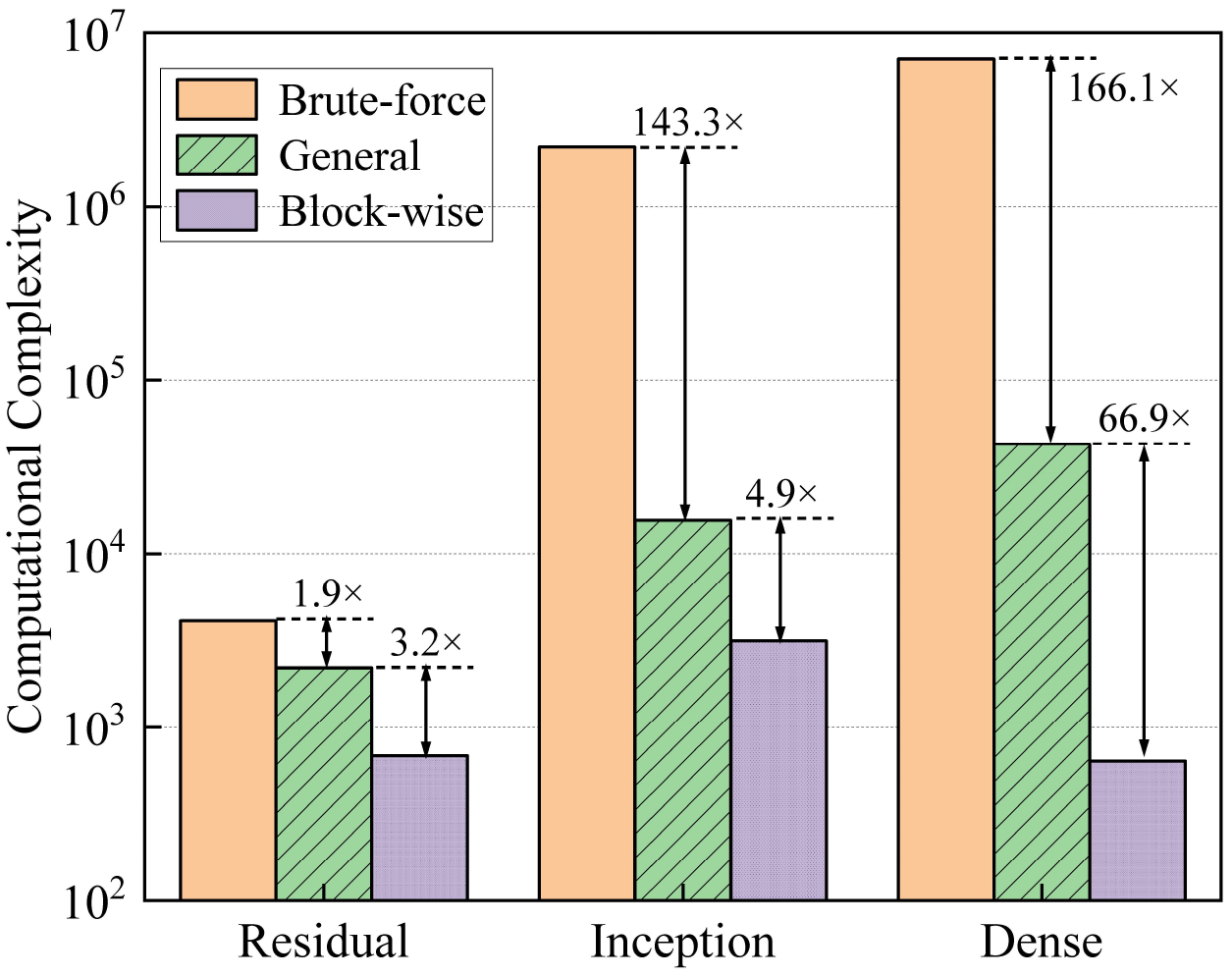} 
        \caption{Computational complexity}
    \end{subfigure}
    \hfill
    \begin{subfigure}[b]{0.24\textwidth}
        \includegraphics[width=\textwidth]{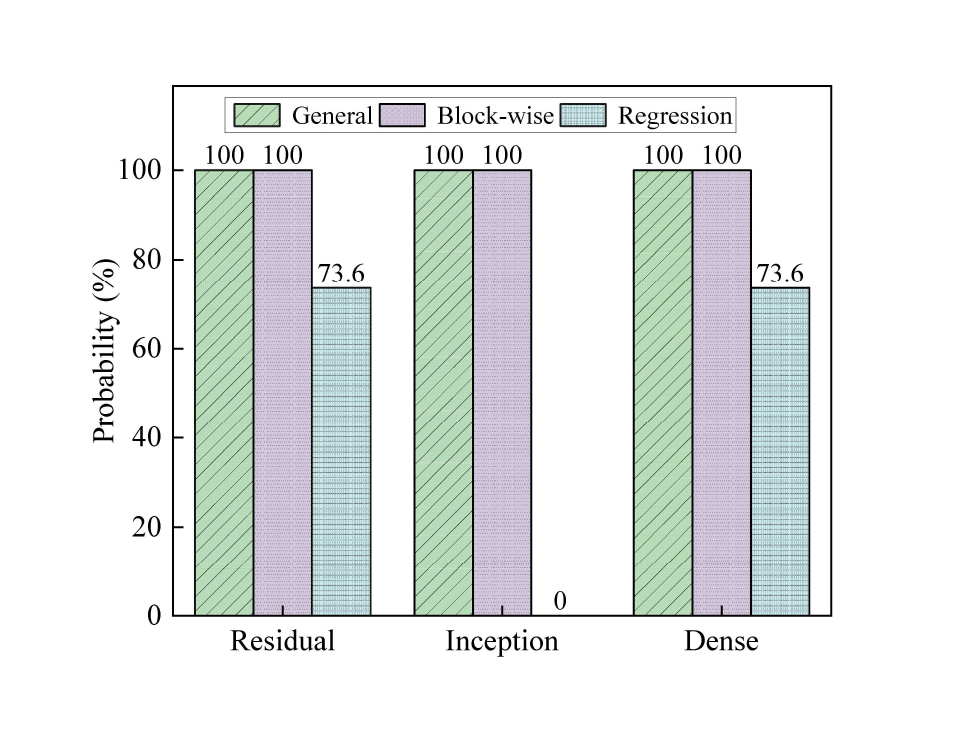} 
        \caption{Probability of the optimal cut}
    \end{subfigure}
    \caption{Performance comparison among the proposed algorithms and benchmarks on networks with a single block.}
    \label{fig: complexity and propagation in DNN block}
\end{figure}

\begin{figure}[t] 
	\renewcommand{\figurename}{Fig.}
	\centering
	\includegraphics[width=0.32\textwidth]{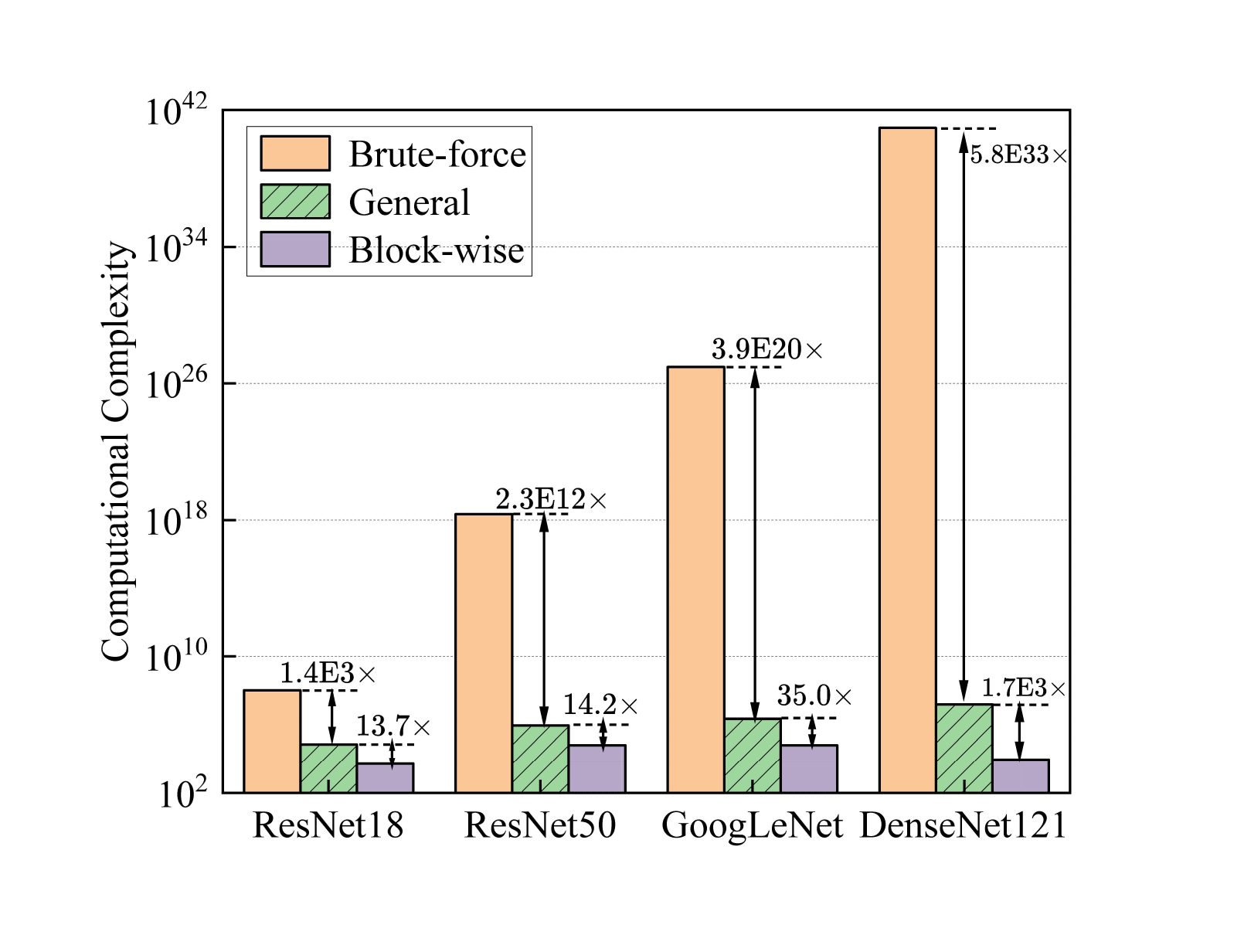}
	\caption{Computational complexity comparison among the proposed algorithms and benchmarks on different AI models.}
	\label{fig: Computational complexity in AI models}
\end{figure}

\subsubsection{Running Time of the Proposed Algorithms} 
We evaluate the running time of the proposed algorithms on both networks, with single-block and full AI models. Due to the high complexity of full AI models, we only measure the running time of the brute-force search method on simpler block-based networks. All experimental results are averaged over 1,000 simulation runs to ensure statistical reliability.

Figure~\ref{fig: Running time comparison among the proposed algorithms and benchmarks}(a) compares the running time of the proposed algorithms against two baselines on the block-based networks. {Although these networks have relatively simple block structures, the brute-force search method still requires up to seconds of running time. Therefore, the brute-force search method is clearly impractical for practical AI models, since these models (e.g., ResNet18 and GoogLeNet) are significantly more complex than block-based networks. 
In contrast, the proposed algorithms significantly improve the computational efficiency of determining the optimal cut.} The general model partitioning algorithm reduces the running time by $12.1 \times$, $4{,}015.6 \times$, and $9{,}998.4 \times$ across the three networks, respectively, compared to the brute-force search method. Furthermore, the block-wise model partitioning algorithm achieves additional reductions of $1.2 \times$, $1.9 \times$, and $3.1 \times$ compared to the general model partitioning algorithm. {Both algorithms can determine the optimal cut for SL within milliseconds.}

Figure~\ref{fig: Running time comparison among the proposed algorithms and benchmarks}(b) presents the running time of the proposed algorithms on four full AI models. The results show that the block-wise model partitioning algorithm consistently achieves lower running time than the general model partitioning algorithm across all AI models. {Moreover, both algorithms can still determine the optimal cut within milliseconds, even for these non-linear AI models. This is well below the response latency tolerance of 200~ms~\cite{tan2018Supporting, xu2025bridging}, indicating that the decision-making process can be completed in a timely manner even in dynamic environments.} Although the regression method achieves the lowest overall running time, it fails to identify the optimal model partition, resulting in a significantly higher training delay. The limitations of the regression method will be presented in the following subsection.

{To quantitatively compare the running time of the proposed algorithms with the model training delay per iteration, we conducted measurements on our hardware testbed. The experimental results are summarized in Table~\ref{tab: Running time and training delay}. It can be observed that the average running time required to determine the cut is on the millisecond level, whereas the average training delay per iteration is on the minute level. Therefore, the running time is several orders of magnitude smaller than the training delay and can be safely ignored in practical training processes.
}

\begin{figure}[t]
    \centering
    \renewcommand{\thesubfigure}{(\alph{subfigure})}
    \begin{subfigure}[b]{0.24\textwidth}
        \includegraphics[width=\textwidth]{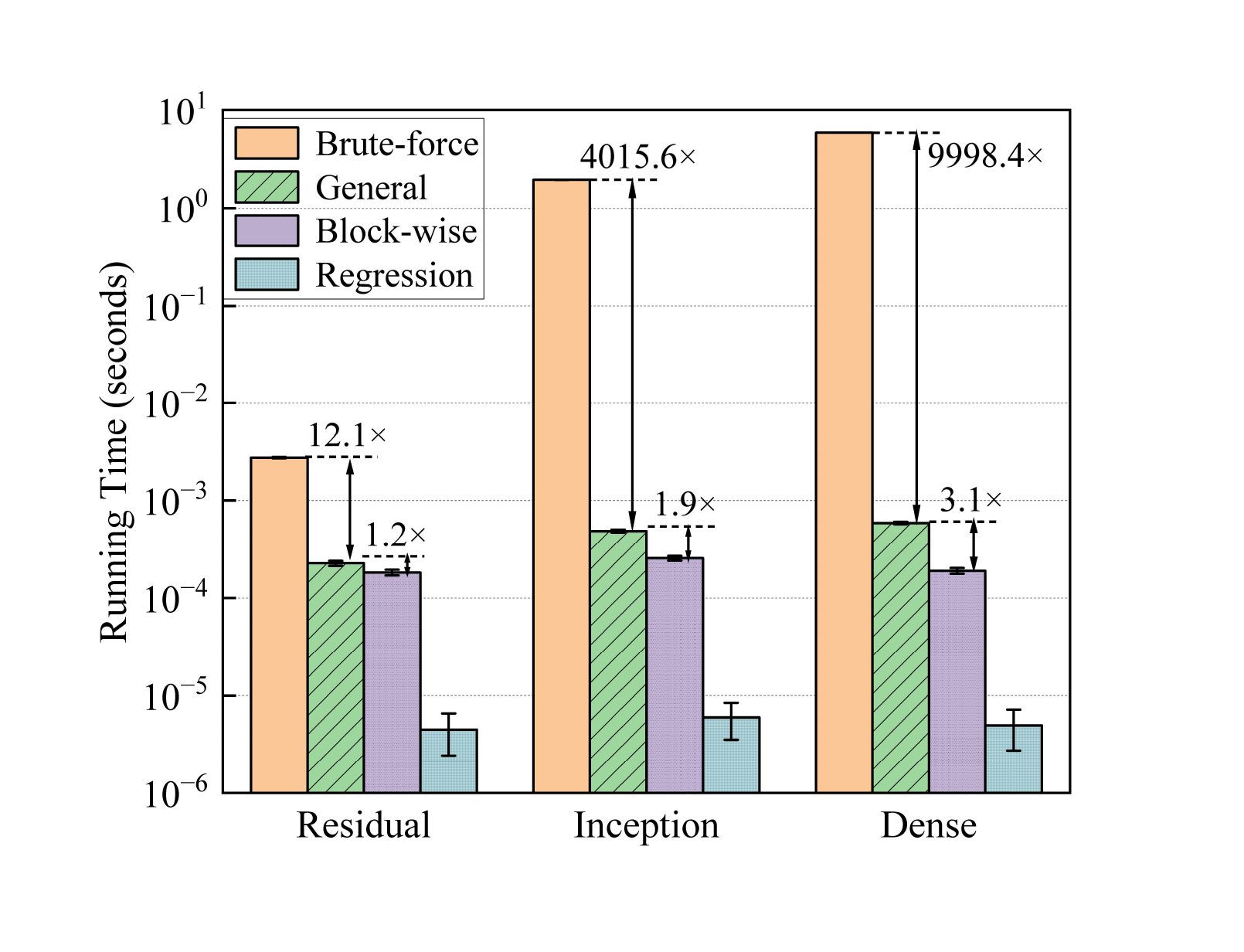} 
        \caption{Network with a block}
    \end{subfigure}
    \hfill
    \begin{subfigure}[b]{0.24\textwidth}
        \includegraphics[width=\textwidth]{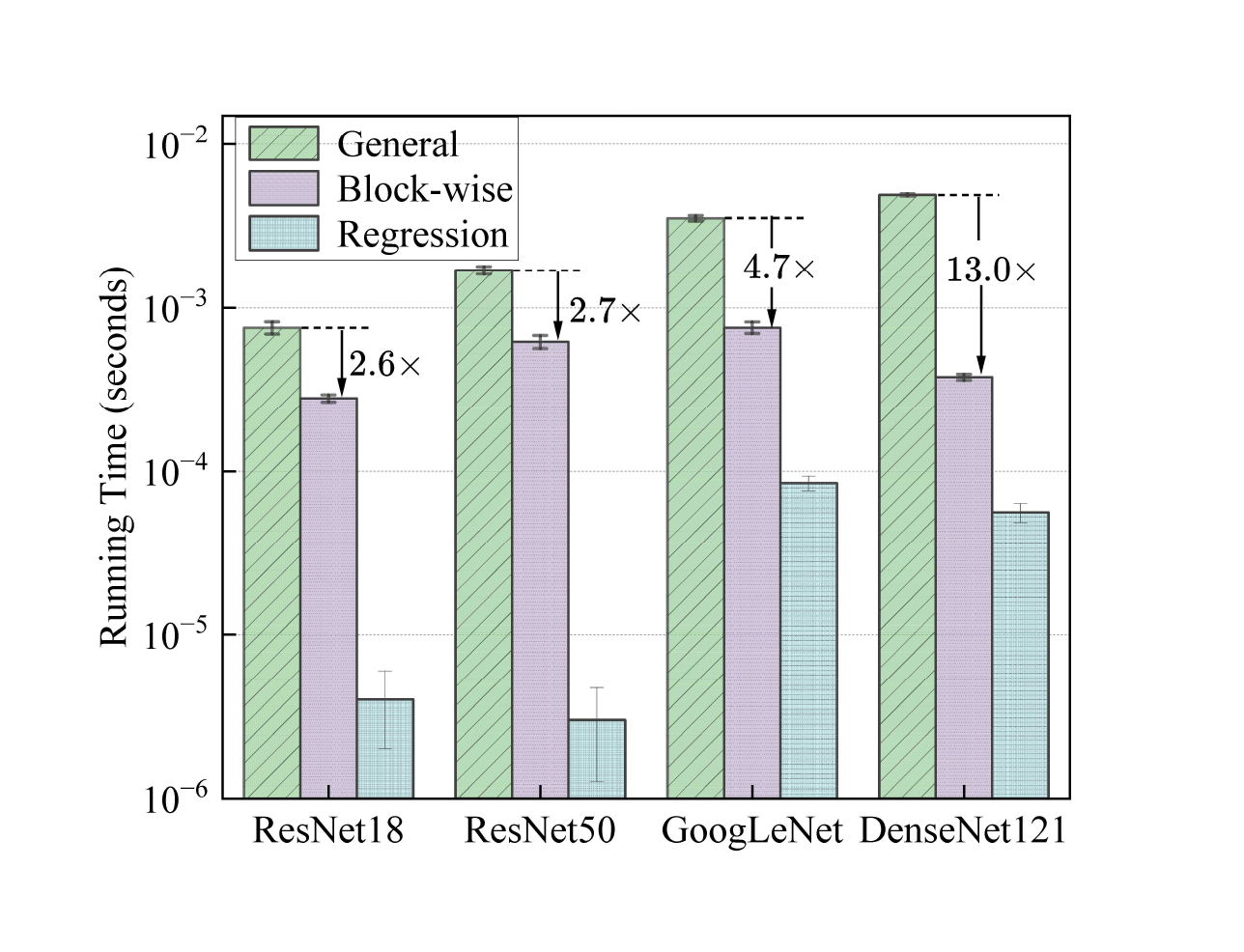} 
        \caption{AI model}
    \end{subfigure}
    \caption{Running time comparison among the proposed algorithms and benchmarks.}
    \label{fig: Running time comparison among the proposed algorithms and benchmarks}
\end{figure}

\begin{table}[t]
\centering
\caption{{Running Time and Training Delay Comparison.}}
\label{tab: Running time and training delay}
\renewcommand{\arraystretch}{1.2} 
\setlength{\tabcolsep}{4.8pt}
\centering
\begin{tabular}{ccccccc}
\hline
\multicolumn{2}{c}{\multirow{2}{*}{AI Model}} & \multicolumn{3}{c}{Running Time (sec.)} & \multicolumn{2}{c}{\multirow{2}{*}{\makecell{Training Delay \\ Per Iteration (sec.)}}} \\
\cline{3-5} 
\multicolumn{2}{c}{}                       & General                    &      & Block-wise        & \multicolumn{2}{c}{} \\
\hline
ResNet18    &                               & $7.58 \times 10^{-4}$     &      & $2.81 \times 10^{-4}$    & \multicolumn{2}{c}{89.28}   \\
ResNet50    &                               & $1.70 \times 10^{-3}$      &      & $6.22 \times 10^{-4}$    & \multicolumn{2}{c}{150.64}   \\
GoogLeNet   &                               & $3.53 \times 10^{-3}$     &      & $7.58 \times 10^{-4}$    & \multicolumn{2}{c}{66.44}   \\
DenseNet121 &                               & $4.91 \times 10^{-3}$     &      & $3.79 \times 10^{-4}$    & \multicolumn{2}{c}{78.90}   \\
\hline
\end{tabular}
\end{table}

\subsection{Proposed Scheme Performance in Dynamic Edge Networks}

\subsubsection{Simulation Setup} We evaluate the proposed algorithms using a custom-designed edge network simulator. The simulator's key components are listed below.

\textit{Network Topology:} The proposed solution is tested in two representative edge network environments: mmWave and sub-6 GHz networks. Each edge network consists of a single base station equipped with a server and connected to 20 mobile devices via wireless links. For wireless data transmission, the n257 band is used in the mmWave network, and the n1 band is used in the sub-6 GHz network, following 3GPP specifications~\cite{3gpp2024user, hakim2024incident}.

\begin{figure}[t] 
	\renewcommand{\figurename}{Fig.}
	\centering
	\includegraphics[width=0.40\textwidth]{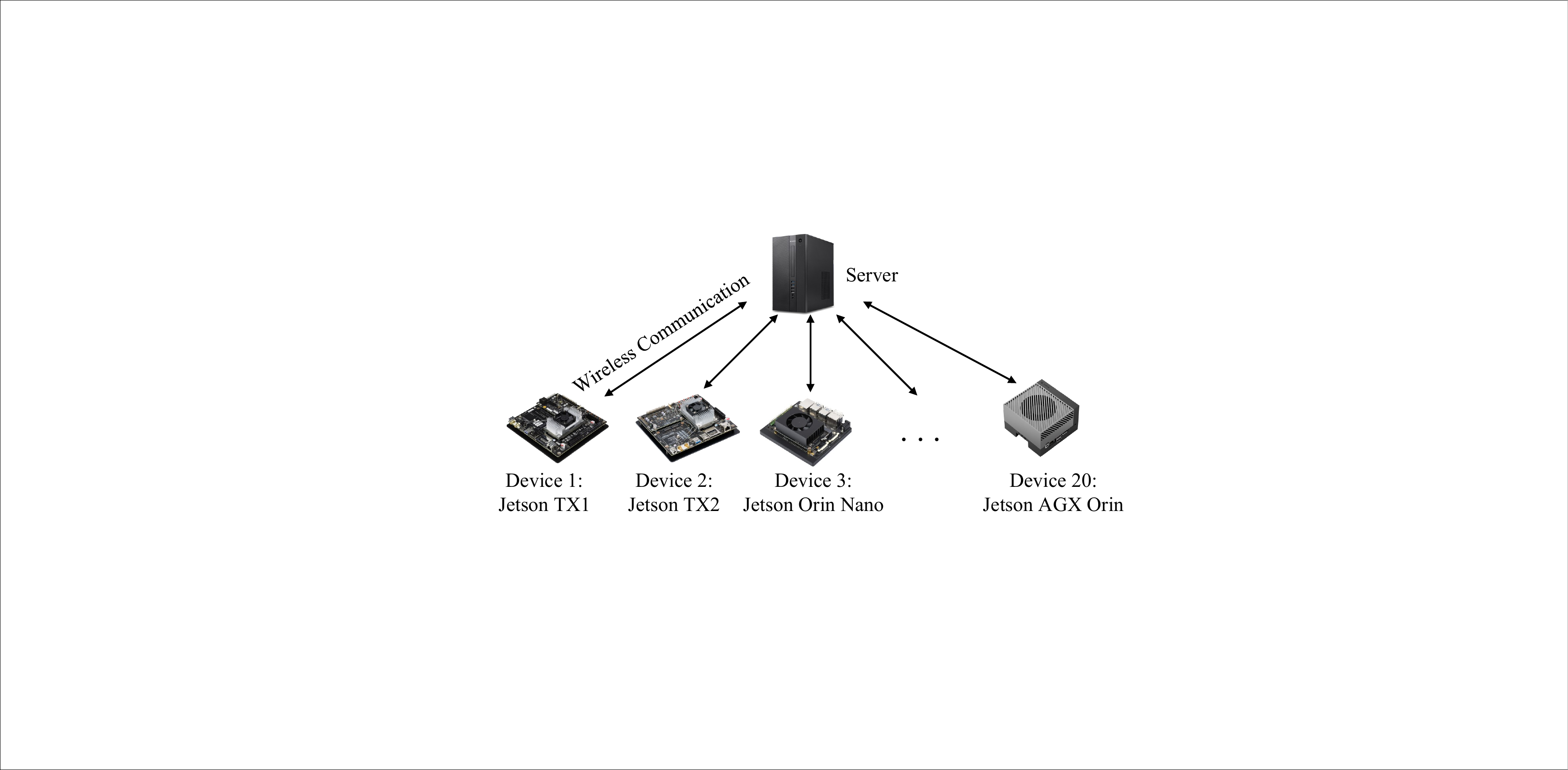}
	\caption{Prototype system of SL.}
	\label{fig: Platform}
\end{figure}

\begin{table*}[h]
\centering
    \caption{Training Delay Comparison among Different Schemes.}
\label{tab: Performance comparison}
\centering
\begin{tabular}{cclcclccl}
\hline
\multirow{3}{*}{Model} &
  \multirow{3}{*}{Method} &
   &
  \multicolumn{5}{c}{Training Delay (minutes)} &
   \\ \cline{3-9} 
                                      &                &  & \multicolumn{2}{c}{CIFAR-10} &  & \multicolumn{2}{c}{CIFAR-100} &  \\ \cline{4-5} \cline{7-8}
 &
   &
   &
  \begin{tabular}[c]{@{}c@{}}IID\\ (Accuracy $95 \%$)\end{tabular} &
  \begin{tabular}[c]{@{}c@{}}non-IID\\ (Accuracy $95\%$)\end{tabular} &
   &
  \multicolumn{1}{c}{\begin{tabular}[c]{@{}c@{}}IID\\ (Accuracy $79 \pm 1 \%$)\end{tabular}} &
  \multicolumn{1}{c}{\begin{tabular}[c]{@{}c@{}}non-IID\\ (Accuracy $78 \pm 1 \%$)\end{tabular}} &
   \\ \hline
\multirow{5}{*}{GoogLeNet}   & OSS &  & 2,872.04 (\textbf{1.33$\times$})          & 3,839.34 (\textbf{1.33$\times$})          &  & 3,201.44 (\textbf{1.29$\times$})           &  3,601.12 (\textbf{1.29$\times$})                 &  \\
                                      & Device-only    &  & 3,265.28 (\textbf{1.51$\times$})          & 4,367.28 (\textbf{1.51$\times$})          &  & 3,851.60 (\textbf{1.55$\times$})            &  4,105.14 (\textbf{1.54$\times$})                 &  \\
                                      & Regression    &  & 3,489.76 (\textbf{1.61$\times$})          & 4,737.32 (\textbf{1.64$\times$})          &  & 3,911.03 (\textbf{1.57$\times$})           &  4,390.73 (\textbf{1.65$\times$})                &  \\
                                      & Proposed       &  & 2,164.36          & 2,891.81          &  & 2,489.61            &  2,660.29                 &  \\ \hline
                                      
\multirow{5}{*}{ResNet18}  & OSS &  & 2,905.48 (\textbf{1.28$\times$})            & 3,272.18 (\textbf{1.29$\times$})      &  & 3,127.26 (\textbf{1.25$\times$})           & 3,327.46 (\textbf{1.34$\times$})        &  \\
                                      & Device-only   &  & 3,331.11 (\textbf{1.47$\times$})         & 3,640.04 (\textbf{1.43$\times$})      &  & 3,407.30 (\textbf{1.36$\times$})          & 3,361.40 (\textbf{1.35$\times$})            &  \\
                                      & Regression    &  & 3,201.30 (\textbf{1.41$\times$})         & 3,502.76 (\textbf{1.38$\times$})         &  & 3,440.95  (\textbf{1.38$\times$})          & 3,318.70 (\textbf{1.34$\times$})              &  \\
                                      & Proposed       &  & 2,265.09          & 2,540.26       &  & 2,499.24           & 2,484.79          &  \\ \hline
                                      
\multirow{5}{*}{ResNet50}             & OSS             &  & 3,242.20 (\textbf{1.33$\times$})         & 4,208.29 (\textbf{1.16$\times$})         &  &  2,790.26 (\textbf{1.28$\times$})            & 3,842.83  (\textbf{1.36$\times$})               &  \\
                                      & Device-only    &  & 3,685.03 (\textbf{1.51$\times$})         & 5,586.54 (\textbf{1.56$\times$})         &  & 3,187.27 (\textbf{1.46$\times$})            & 4,057.17 (\textbf{1.44$\times$})               &  \\
                                      & Regression    &  & 3,639.25 (\textbf{1.49$\times$})         & 5,490.38 (\textbf{1.52$\times$})         &  & 3,113.14 (\textbf{1.43$\times$})           & 4,038.32 (\textbf{1.43$\times$})               &  \\
                                      & Proposed       &  & 2,445.16         & 3,611.98          &  & 2,179.92              & 2,819.09                  &  \\ \hline
                                      
\multirow{5}{*}{DenseNet121}   & OSS &  & 3,717.68 (\textbf{1.16$\times$})         & 4,215.98 (\textbf{1.15$\times$})         &  & 3,002.23  (\textbf{1.28$\times$})        &  4,300.22 (\textbf{1.33$\times$})               &  \\
                                      & Device-only    &  & 4,940.63 (\textbf{1.55$\times$})        & 5,744.91 (\textbf{1.57$\times$})         &  & 3,420.54 (\textbf{1.46$\times$})           & 5,021.96 (\textbf{1.55$\times$})              &  \\
                                      & Regression    &  & 4,921.35 (\textbf{1.54$\times$})          & 5,702.53 (\textbf{1.56$\times$})          &  & 3,374.06 (\textbf{1.44$\times$})            & 4,954.29 (\textbf{1.53$\times$})                &  \\
                                      & Proposed       &  & 3,196.94         & 3,652.77          &  & 2,341.60          & 3,232.60                &  \\ \hline
\end{tabular}
\end{table*}

\textit{Device Mobility:} These devices, located within the base station's coverage area, request access and begin the device-side model training task at the start of each experimental trial.
Each device continues to run the task while moving along a predefined trajectory at 30 \textit{km/h}. The link bitrate is converted by the new radio channel quality indicator to the modulation and coding scheme mapping table~\cite{3gpp2022phsical}. 
Although the channel and link scheme does not account for complex propagation effects, it effectively simulates real-world mobile cellular dynamics, making it suitable for testing model partition. {At each training round, the edge device closest to the base station is selected to participate in the task. To maintain fairness among devices, once a device is selected within an epoch, it cannot be selected again during the same epoch.}

\textit{Channel Condition:} Taking directional antenna gain into account, the transmit power is calculated as $P = P_e - 10 \log_{10} N$, where $P_e$ is the average effective isotropically radiated power (EIRP), and $N$ is the number of beams. The server's average EIRP is fixed at 40 dBm for sub-6 GHz networks and 50 dBm for mmWave networks~\cite{wang2025Performance}. The number of beams is set to 16 for sub-6 GHz networks and 64 for mmWave networks. The large-scale path loss is defined as 
\begin{equation}\label{equ: path loss}
 \begin{split}
   PL(dB) = 32.5 + 20 \log_{10}(f) + 10\eta \log_{10}(d) + \chi,
 \end{split}
\end{equation}
where $f$, $\eta$, $d$, and $\chi$ represent the carrier frequency, path-loss exponent, transmission distance, and shadow fading, respectively. The shadow fading follows $\mathcal{N}(0, \sigma^2)$. {Furthermore, a small-scale fading is incorporated into the channel. Specifically, the effective path loss under Rayleigh fading is given by
\begin{equation}
    PL_{\text{small}}(dB) = PL(dB) - 10\log_{10}(\psi),
\end{equation}
where $\psi$ denotes an independent random channel fading factor following an exponential distribution with unit mean~\cite{huang2020deep}.
}
In this paper, the channel condition is set to three states: \textit{Good}, \textit{Normal}, and \textit{Poor}, corresponding to $\sigma$ values of 2 dB, 4 dB, and 6 dB, respectively. 

{In practical wireless systems, the base station continuously collects channel quality information (e.g., channel quality indicators (CQI) and buffer status reports (BSR)) as part of standard radio resource management procedures. Our method leverages these existing measurements to estimate the uplink and downlink transmission rates, without requiring any additional signaling mechanisms. According to measurements in the 5G radio access network (RAN), the average packet round-trip time (RTT), which reflects the latency of channel estimation and link adaptation, is approximately 9.23~ms~\cite{xu2025bridging}. This indicates that the communication overhead for collecting network state information is relatively small.}

\textit{Computation Delay Profiles:} To accurately measure computation delay,  develop a standalone Python-based latency profiling tool for \textit{PyTorch}. This tool traverses the AI model and registers forward and backward hooks on each layer, following the method introduced in~\cite{wang2022hivemind}.  These hooks record per-layer processing times during training, enabling fine-grained profiling of training delays. For communication delay estimation, we use PyTorch’s built-in \texttt{torchstat} module to record each layer’s parameter size and output size. All training and profiling experiments are performed on graphics processing units (GPUs). 

Our prototype system is illustrated in Fig.~\ref{fig: Platform}, consisting of a server and 20 devices configured as follows:  (\textit{i}) a personal computer (PC) with $1 \times$ Nvidia RTX A6000 GPU; (\textit{ii}) five Jetson TX1, each with 256-core Nvidia Maxwell GPU; (\textit{iii}) five Jetson TX2, each with 256-core Nvidia Pascal GPU; (\textit{iv}) five Jetson Orin Nano, each with 1024-core Nvidia Ampere architecture GPU; and (\textit{iv}) five Jetson AGX Orin, each with 2048-core Nvidia Ampere architecture GPU. Through the prototype system, we measure training delay profiles on both the server and heterogeneous mobile devices, demonstrating the system’s ability to adapt to real-world variations in device capabilities. The performance of the proposed solution is compared with three benchmarks:
\begin{itemize}
\item \textbf{Optimal static partition (OSS)}~\cite{xu2023accelerating}, which refers to the static cut that minimizes training latency among all possible fixed cuts.
\item \textbf{Device-only}, where the entire AI model is trained on the device, while the server is only responsible for data transmission.
\item \textbf{Regression}~\cite{wen2025training}, as mentioned in the previous subsection.
\end{itemize}

\subsubsection{Impact of Channel Conditions} To evaluate the performance of the proposed solution under dynamic wireless environments, we train GoogLeNet over 300 simulation runs across three channel conditions.

Figure~\ref{fig: delay with three channels} illustrates the model training delay per epoch for each method under both sub-6 GHz and mmWave networks. The regression method performs poorly due to its suboptimal model partition. In contrast, the proposed solution consistently achieves the lowest training delay across all channel conditions. In the sub-6 GHz network, the proposed solution reduces training delay by 11.40\%, 12.38\%, and 19.34\% compared to the OSS, device-only, and regression methods, respectively. In the mmWave network, the proposed solution also outperforms all benchmarks, achieving reductions in training delay of 27.38\%, 36.09\%, and 38.61\%, respectively. These results demonstrate that the proposed solution can rapidly determine the optimal model partition under real-time network conditions, particularly in highly dynamic environments.

{Figure~\ref{fig: delay_mmWave_Rayleigh} shows the performance under Rayleigh fading channels in the mmWave network. As we can see, the training delay per epoch of the OSS method is significantly affected by Rayleigh fading, exhibiting large fluctuations. In contrast, the proposed solution remains relatively stable under the same conditions. This is because the transmission rate under Rayleigh fading varies much more dramatically than that under large-scale path loss. Since the OSS method uses a fixed cut, it cannot adapt to rapid channel variations, leading to substantial fluctuations in transmission delay. In contrast, the proposed solution dynamically adjusts the cut according to channel conditions, thereby maintaining a more stable training delay per epoch even in the presence of Rayleigh fading. Furthermore, under Rayleigh fading channels, the proposed solution consistently outperforms all baselines across different channel conditions.}

\begin{figure}[t]
    \centering
    \renewcommand{\thesubfigure}{(\alph{subfigure})}
    \begin{subfigure}[b]{0.24\textwidth}
        \includegraphics[width=\textwidth]{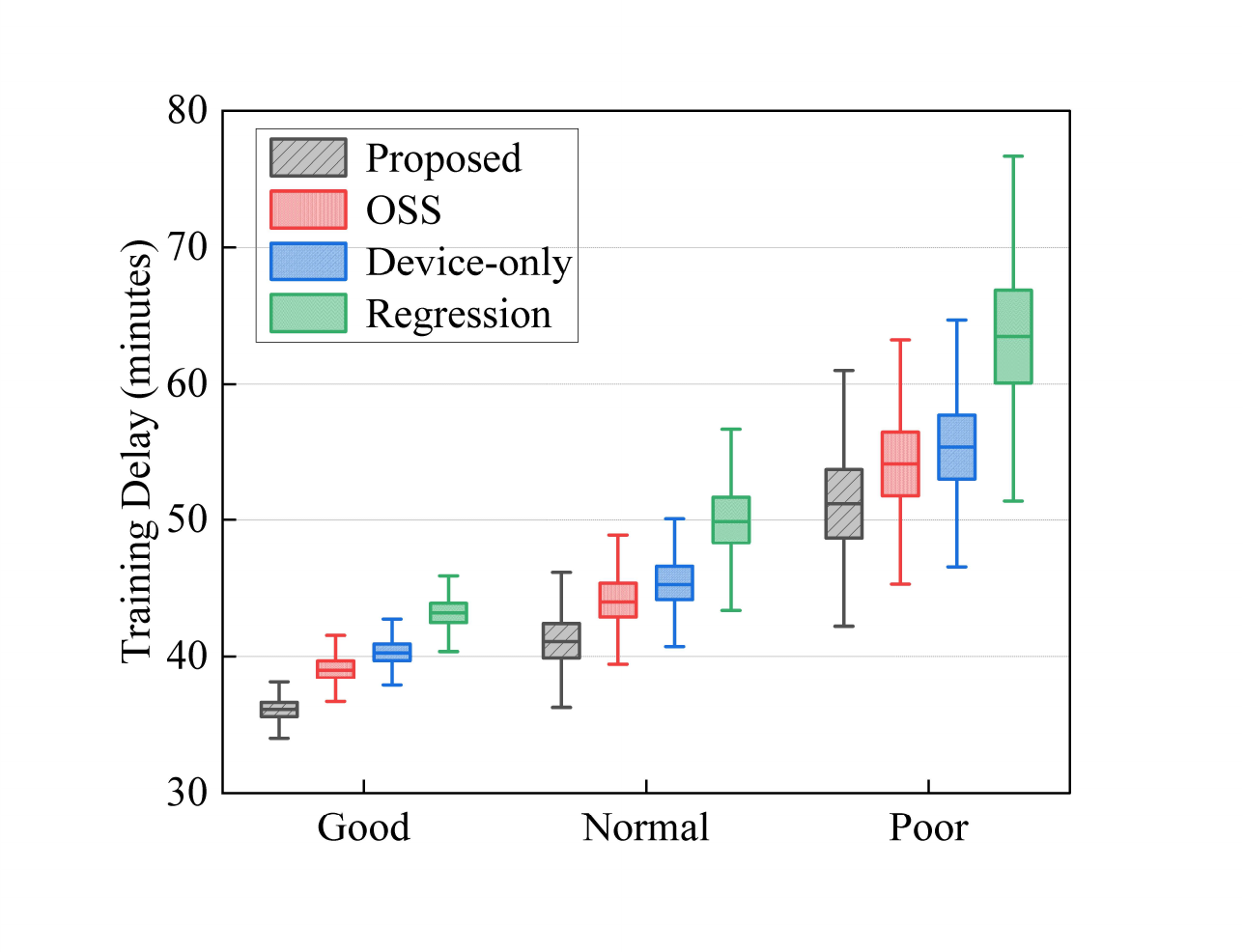} 
        \caption{Sub-6 GHz network}
    \end{subfigure}
    \hfill
    \begin{subfigure}[b]{0.24\textwidth}
        \includegraphics[width=\textwidth]{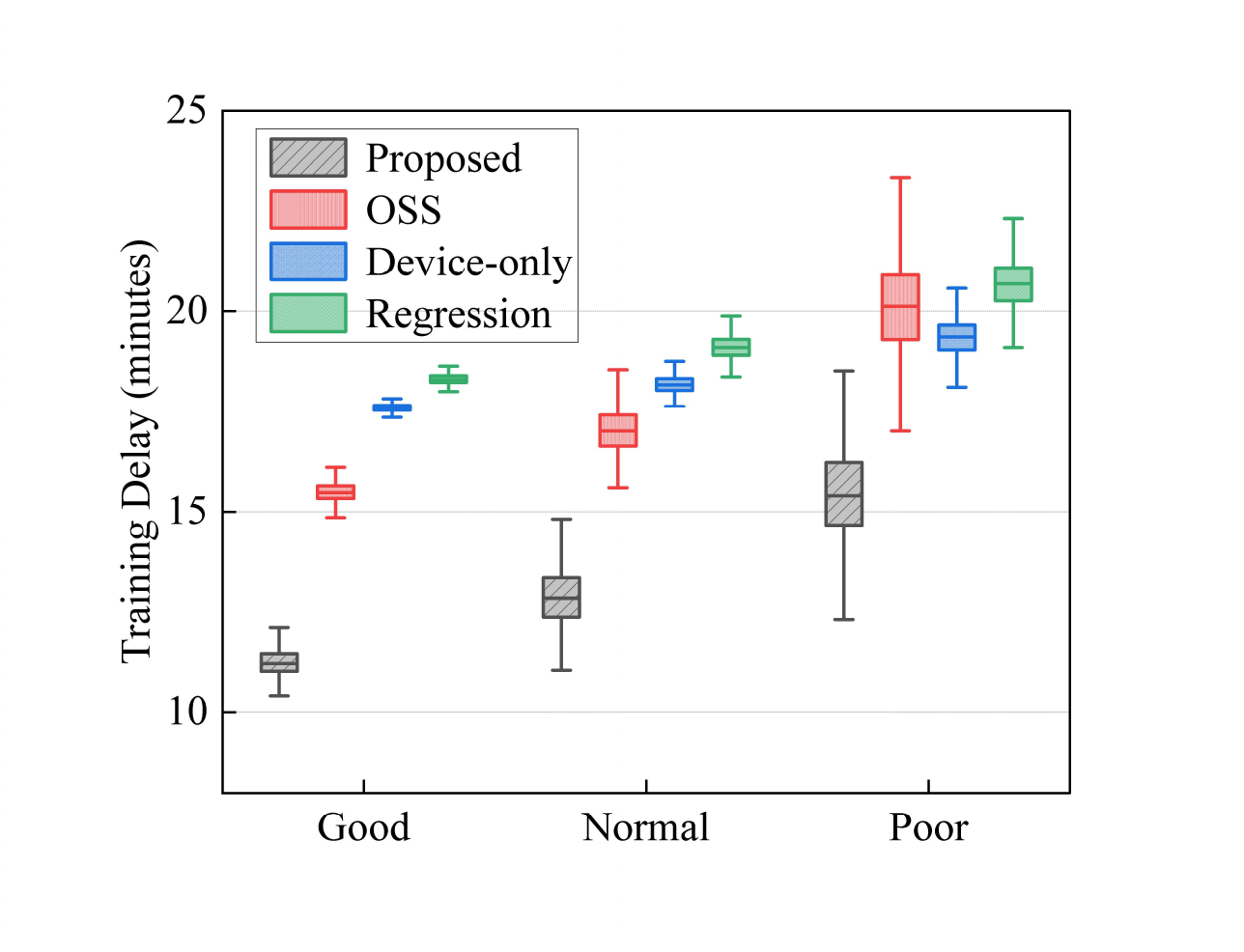} 
        \caption{mmWave network}
    \end{subfigure}
    \caption{{Training delay per epoch of the proposed solution and benchmarks under large-scale path loss.}}
    \label{fig: delay with three channels}
\end{figure}

\begin{figure}[t] 
	\renewcommand{\figurename}{Fig.}
	\centering
	\includegraphics[width=0.25\textwidth]{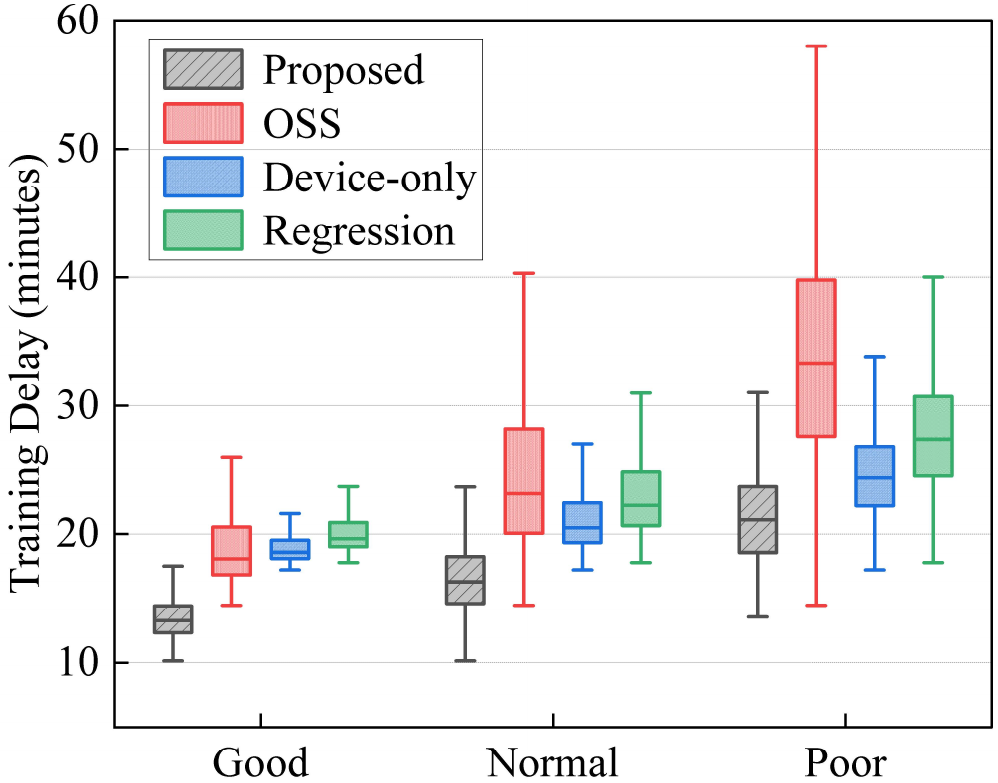}
	\caption{{Training delay per epoch of the proposed solution and benchmarks in the mmWave network under a Rayleigh fading channel.}}
	\label{fig: delay_mmWave_Rayleigh}
\end{figure}

\subsubsection{Impact of  Data Distribution} To evaluate the performance of the proposed solution under different data distributions, we train GoogLeNet and compare it with four benchmarks: \textit{central}, which runs the entire model on the server, \textit{OSS}, \textit{device-only}, and \textit{regression}. The experiments are conducted in the mmWave network with normal channel conditions. Two types of data distributions are considered, i.e., independent and identically distributed (IID) and non-IID. In the IID setting, each device receives an equal number of samples from each class in the dataset.
In this experiment, training delay is defined as the time required for the model to reach 95\% accuracy under each algorithm.
As shown in Fig.~\ref{fig: accuracy in iid and non-iid}(a), the proposed solution reduces the training delay by 37.96\%, 26.22\%, and 24.62\% compared to the regression, device-only, and OSS benchmarks, respectively.

In the non-IID dataset, we employ the method described in \cite{feng2022mobility} to synthesize data for different devices. Specifically, each device samples data from a subset of categories, where the category proportions $\mathcal{Q} = {q_1, q_2, ..., q_M}$ follow a Dirichlet distribution, i.e., $\mathcal{Q} \sim Dir(\gamma p)$. Here, $p$ denotes the label distribution, and $\gamma$ controls the degree of data heterogeneity across devices. In this experiment, we set $\gamma = 0.5$. As shown in Fig.~\ref{fig: accuracy in iid and non-iid}(b), the proposed solution reduces the training delay by 38.95\%, 33.79\%, and 24.68\% compared with the regression, device-only, and OSS benchmarks, respectively. {In the device-only method, although only a single transmission is required after completing the full local iteration to upload the model parameters to the server, the overall training delay remains relatively high. This is because mobile devices typically have limited computational capability, leading to longer local training delay compared to the proposed solution. Moreover, the size of the transmitted model parameters is larger than that of the smashed data at the optimal cut determined by the proposed solution. As a result, both computation and communication delays are increased in the device-only approach.}
This demonstrates that under non-IID data conditions, the proposed solution offers greater advantages, highlighting its robustness to heterogeneous data distributions.

\begin{figure}[t]
    \centering
    \renewcommand{\thesubfigure}{(\alph{subfigure})}
    \begin{subfigure}[b]{0.24\textwidth}        \includegraphics[width=\textwidth]{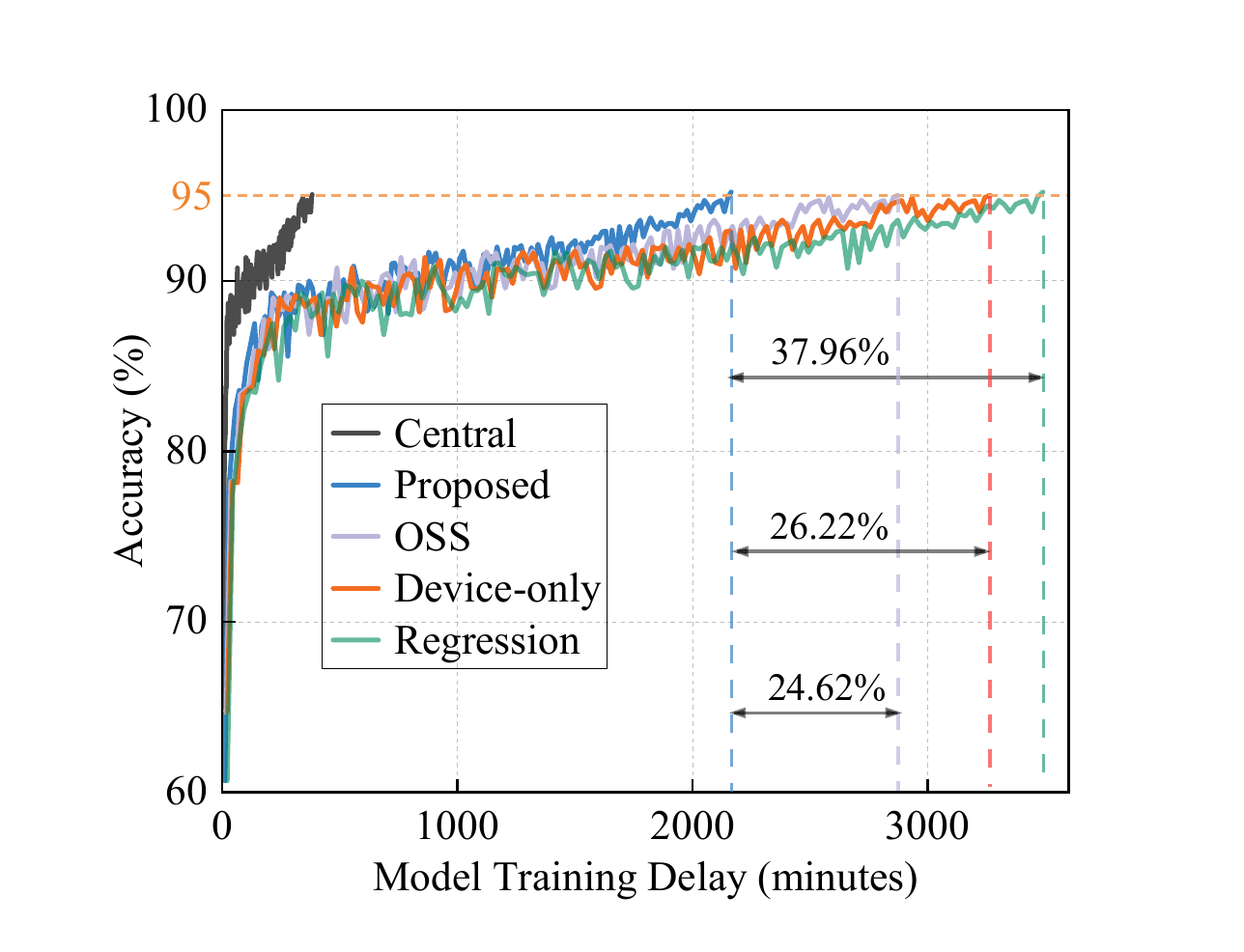} 
        \caption{IID}
    \end{subfigure}
    \hfill
    \begin{subfigure}[b]{0.24\textwidth}
        \includegraphics[width=\textwidth]{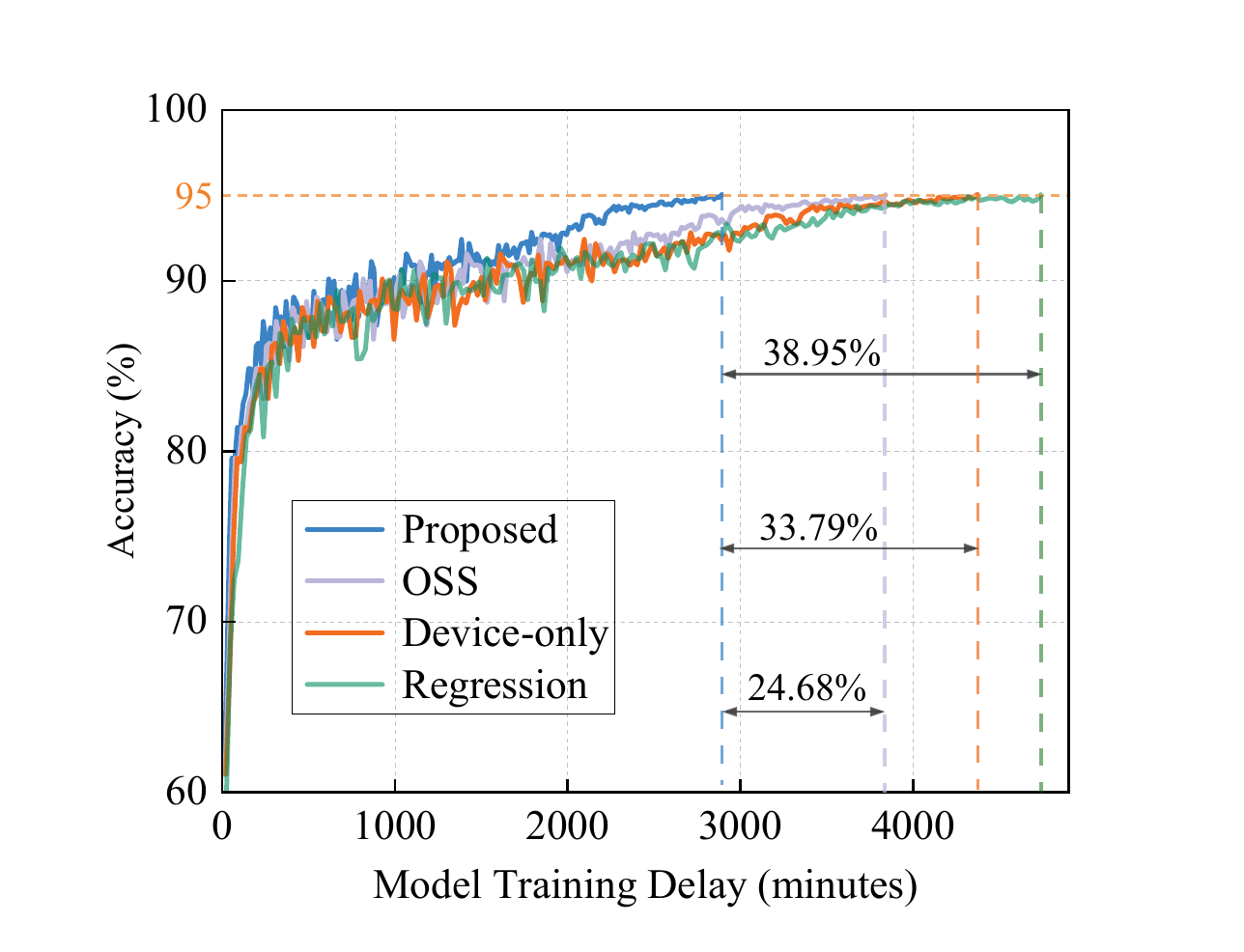} 
        \caption{non-IID}
    \end{subfigure}
    \caption{Overall training delay comparison among the proposed solution and benchmarks when training GoogLeNet.}
    \label{fig: accuracy in iid and non-iid}
\end{figure}

\subsubsection{Impact of AI Models}  All models are trained over the mmWave network with normal channel conditions.
To ensure a fair comparison, all schemes are required to achieve their predefined accuracy thresholds using the same datasets, data distributions, and AI models. These thresholds may vary slightly across models due to inherent differences in their performance. For example, on the CIFAR-100 dataset with a non-IID distribution and ResNet18, the accuracy threshold is set to 77\%, whereas with ResNet50 under the same conditions, the threshold is 78\%. 

The overall training delay for each method to reach its accuracy threshold is summarized in Table~\ref{tab: Performance comparison}. The proposed solution substantially reduces training delay compared with benchmarks across all AI models, datasets, and data distributions. For example, when training DenseNet121 on the CIFAR-100 dataset with a non-IID data distribution, the proposed solution reduces the overall training delays by 1.33×, 1.55×, and 1.53× compared with the OSS, device-only, and regression methods, respectively. {In addition, we further evaluate the performance on GPT-2 under the mmWave network with normal channel conditions. GPT-2 is trained on the CARER dataset~\cite{saravia2018carer} with a non-IID data distribution. As illustrated in Fig.~\ref{fig: GPT-2_accuracy_non-iid}, the proposed solution reduces the overall training delay by 8.62\%, 23.48\%, and 73.42\% compared with the OSS, regression, and device-only methods, respectively.}
These performance gains are primarily due to the proposed solution’s ability to select the optimal model partition that minimizes overall training delay, thereby mitigating resource bottlenecks and accelerating the training process.

\begin{figure}[t] 
	\renewcommand{\figurename}{Fig.}
	\centering
	\includegraphics[width=0.25\textwidth]{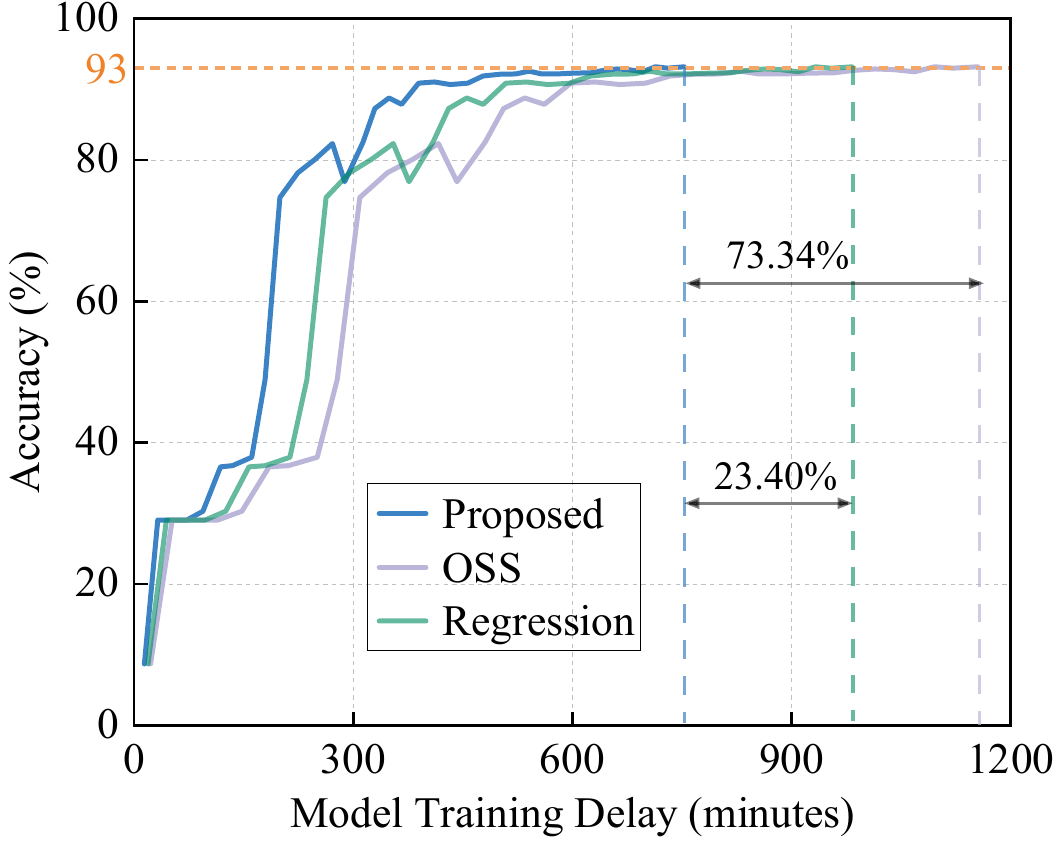}
	\caption{{Overall training delay comparison among the proposed solution and benchmarks when training GPT-2 on the CARER dataset.}}
	\label{fig: GPT-2_accuracy_non-iid}
\end{figure}

\subsubsection{{Impact of Network Sizes}}
{To evaluate the robustness of the proposed approach against network sizes, we measure the training delay under different numbers of devices when training GoogLeNet with a non-IID CIFAR-10 dataset in the mmWave network. As shown in Fig.~\ref{fig: training delay under different numbers of devices}, it can be observed that, when the number of devices is 10 and 40, the proposed scheme consistently outperforms the baseline methods by achieving lower training delay. Specifically, the proposed solution reduces the training delay by at least 25.68\% and 23.46\% compared with baselines when the number of devices is 10 and 40, respectively.
Despite the increase in the number of devices, the proposed approach maintains fast convergence, demonstrating its robustness against network sizes.}
\begin{figure}[t]
    \centering
    \renewcommand{\thesubfigure}{(\alph{subfigure})}
    \begin{subfigure}[b]{0.24\textwidth}        
        \includegraphics[width=\textwidth]{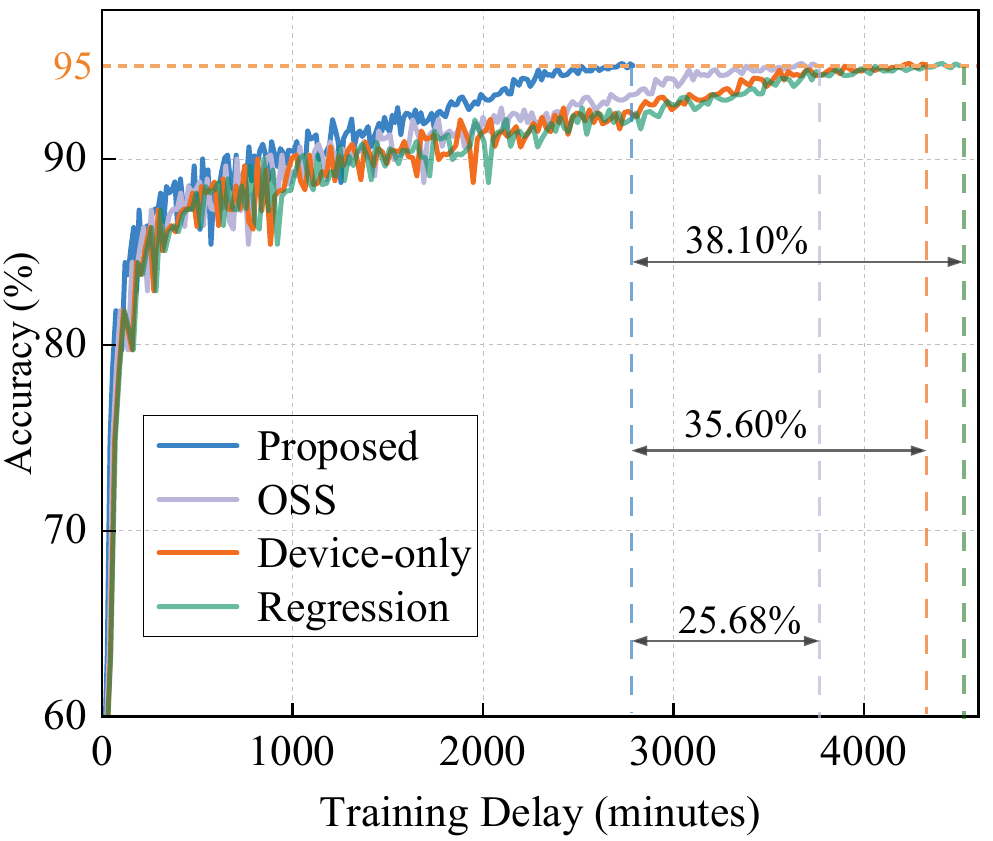} 
        \caption{Training delay (10 devices)}
    \end{subfigure}
    \hfill
    \begin{subfigure}[b]{0.24\textwidth}
        \includegraphics[width=\textwidth]{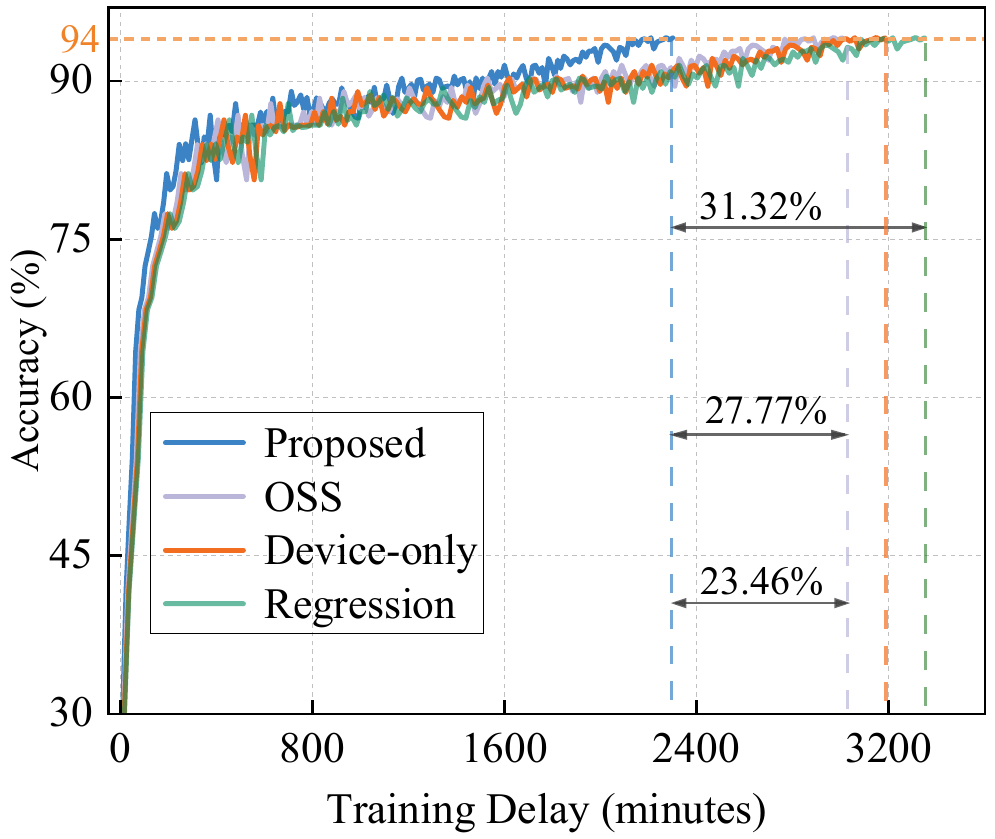} 
        \caption{Training delay (40 devices)}
    \end{subfigure}
    \caption{{Overall training delay comparison under different numbers of devices when training GoogLeNet with a non-IID CIFAR-10 dataset in the mmWave network.}}
    \label{fig: training delay under different numbers of devices}
\end{figure}

\subsubsection{{Communication Overhead}}
{To clarify the communication overhead introduced by model splitting, we compare the device-side computing delay, server-side computing delay, and transmission delay under different methods. Specifically, the results are obtained when training GoogLeNet over the mmWave network with a batch size of 32, where the reported delays correspond to the total computation and transmission time for completing two iterations jointly executed by the device and the server. As shown in Fig.~\ref{fig: GoogLeNet_comm_comp_delay}, the proposed solution reduces the overall training delay by 23.40\% and 73.34\% compared with the regression and OSS methods, respectively. This improvement is attributed to the optimal cut selected by the proposed solution, which effectively reduces the size of the intermediate data transmitted between the device and the server.
In addition, although the device-only method incurs lower transmission delay than the proposed scheme, it suffers from significantly higher device-side computation delay. As a result, its overall training delay is still larger than that of the proposed scheme.}

\begin{figure}[t] 
	\renewcommand{\figurename}{Fig.}
	\centering
	\includegraphics[width=0.3\textwidth]{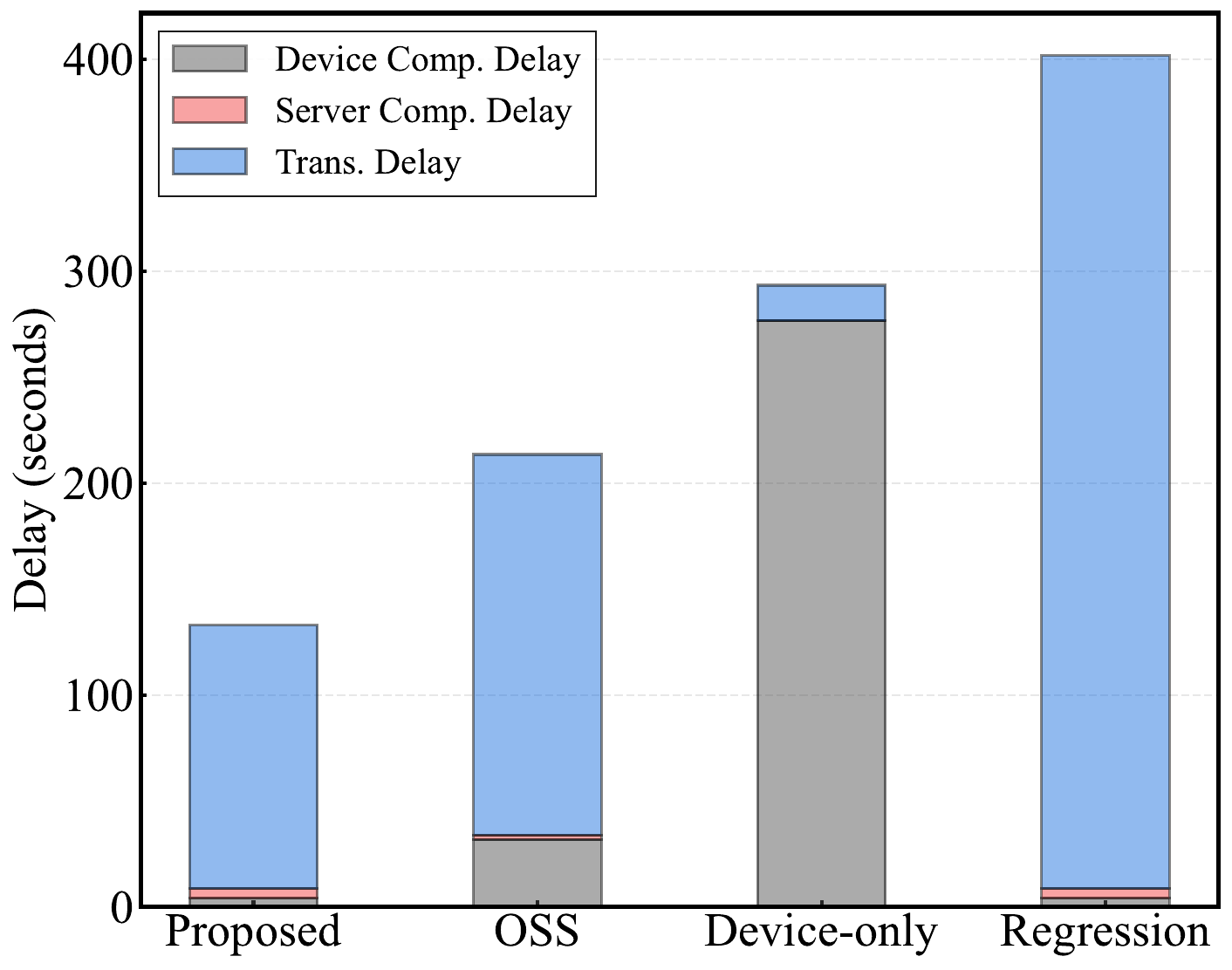}
	\caption{{Computing delay and transmission delay comparison.}}
	\label{fig: GoogLeNet_comm_comp_delay}
\end{figure}

\section{Conclusion}\label{sec: conclusion}
In this paper, we have investigated the optimal model partitioning problem to minimize the training delay of SL in edge networks. We have represented an arbitrary AI model as a DAG and proposed a general model partitioning algorithm to determine the optimal model partition. For AI models with block structures, we have proposed a low-complexity block-wise partitioning algorithm to efficiently determine the optimal model partition. Extensive experimental results have validated that the proposed scheme can find the optimal model partition within milliseconds. Due to fast and accurate identification of the optimal model partition, the proposed solution can be applied to facilitate AI model training in edge networks with a large number of resource-constrained mobile devices. For future work, we will investigate the multi-partition problem of SL in edge networks.

\section*{Appendix A\\ Proof of Theorem 1}
\renewcommand{\theequation}{A.\arabic{equation}}
\setcounter{equation}{0} 
\setcounter{figure}{0}
\renewcommand{\thefigure}{A.\arabic{figure}}

As shown in Fig.~\ref{fig: cuts_cases in AI model}, all possible cuts in an AI model can be categorized into three cases:
\textit{(i) }the cut intersects all outgoing edges originating from a parent vertex, e.g., case 1 in Fig.~\ref{fig: cuts_cases in AI model}(a);
\textit{(ii)} the cut intersects only some of the outgoing edges from a parent vertex, e.g., case 2 in Fig.~\ref{fig: cuts_cases in AI model}(b);
and \textit{(iii)} the cut does not intersect any outgoing edges from a parent vertex, e.g., case 3 in Fig.~\ref{fig: cuts_cases in AI model}(c).
Let $T(\cdot)$ denote the model training delay in the original AI model, calculated as the sum of computation and transmission delays. Let $T_{\mathcal{G}'}(\cdot)$ denote the value of the \textit{s-t} cut in DAG $\mathcal{G}'$, calculated as the sum of the weights of all edges intersected by this cut.
The corresponding training delay under each case is analyzed as follows.

\emph{Case 1:} The model training delay in the AI model is equal to the value of the cut crossing the edge from the auxiliary vertex to its original vertex in $\mathcal{G}'$. For example, the value of cut ${c}_{1}$ in Fig.~\ref{fig: cuts_cases in AI model}(a) is the same as that of cut ${c}_{1'}$ in Fig.~\ref{fig: cuts_cases in DAG-2}(a), i.e., $T(c_1) =  T_{\mathcal{G}'}(c_{1'})$. In addition, the original cut in $\mathcal{G}'$ (e.g., cut ${c}_{1}$ in Fig.~\ref{fig: cuts_cases in DAG-2}(a)) is impossible to be the minimum \textit{s-t} cut, because its value is more than the value of the cut crossing the edge between the auxiliary vertex and its original vertex (e.g., $T_{\mathcal{G}'}({c}_{1'}) < T_{\mathcal{G}'}({c}_{1}) $ in Fig.~\ref{fig: cuts_cases in DAG-2}(a)). In cut ${c}_1$ of Fig.~\ref{fig: cuts_cases in DAG-2}(a), the weights of outgoing edges from vertex $v_1$ are counted three times, while the weight of the outgoing edge from vertex $v_{1'}$ is counted only once in cut ${c}_{1'}$. This leads to $T_{\mathcal{G}'}({c}_{1'}) < T_{\mathcal{G}'}({c}_{1})$. Hence, this case does not require further analysis.

\emph{Case 2:} If the server’s computing capability is not weaker than that of the device, the corresponding model training delay in the AI model cannot be minimal, and the cut in this case cannot be the minimum \textit{s-t} cut in $\mathcal{G}'$.
We first analyze the cut in the AI model. Compared with \textit{case 1}, the model training delay in \textit{case 2} is strictly larger. For instance, as shown in Fig.~\ref{fig: cuts_cases in AI model}(a) and Fig.~\ref{fig: cuts_cases in AI model}(b), the difference in delay between cut $c_2$ and cut $c_1$ is given by
\begin{equation}\label{equ: case 2 - case 1}
  \begin{split}
    T(c_2) - T(c_1) = & N_{loc} \bigg( \xi_{D, v_4} - \xi_{S, v_4} + \frac{a_{v_4}}{R_D} \\
    &+ \frac{a_{v_4}}{R_S} \bigg)  + \frac{k_{v_4}}{R_D} + \frac{k_{v_4}}{R_S}.
  \end{split}
\end{equation}

\begin{figure}[t] 
	\renewcommand{\figurename}{Fig.}
	\centering
	\includegraphics[width=0.22\textwidth]{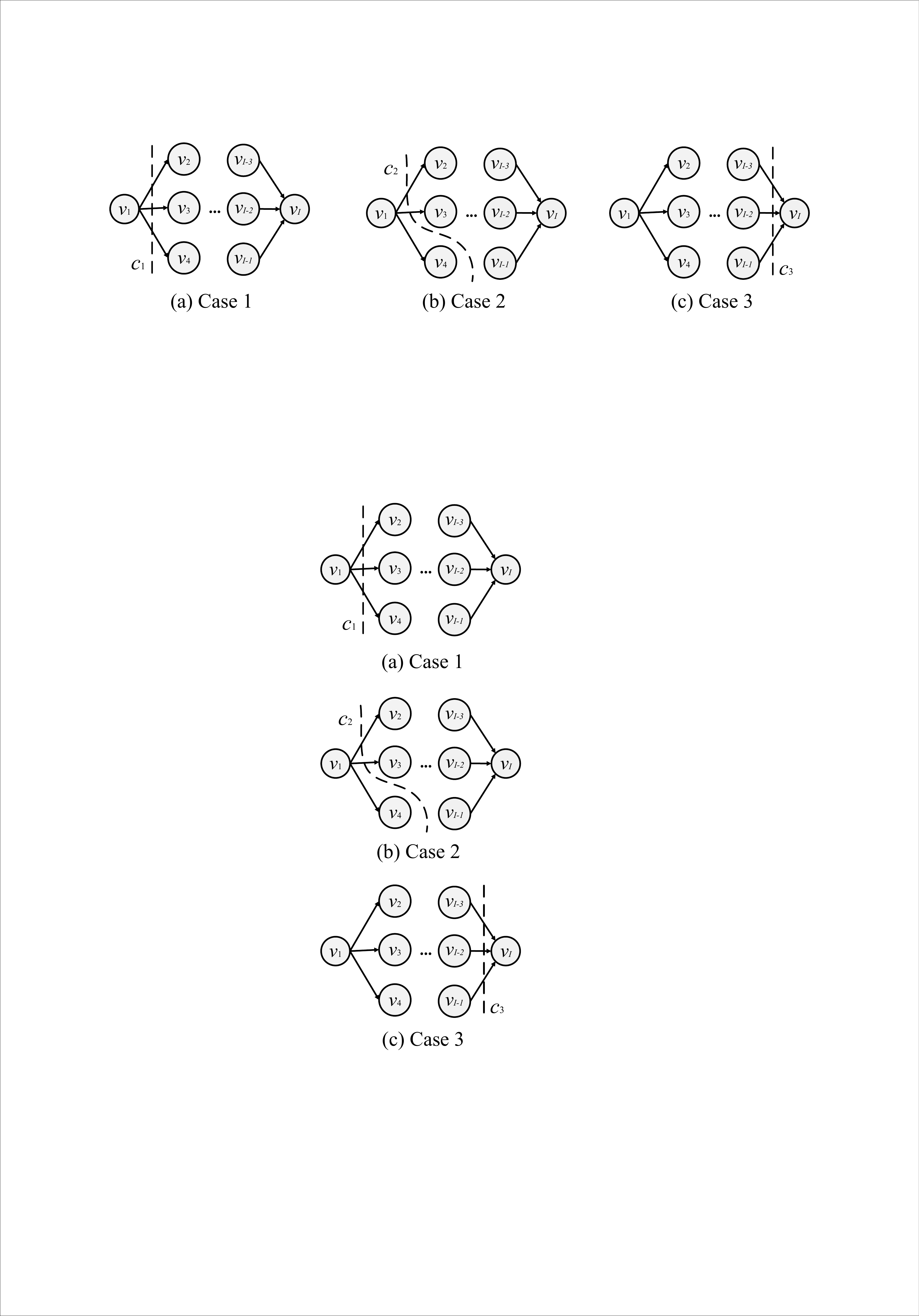}
	\caption{{Three cut cases in the AI model, where each cut’s value is the total training delay.} }
	\label{fig: cuts_cases in AI model}
\end{figure}

Under Assumption 1, we have $\xi_{D, v_4} - \xi_{S, v_4} \ge 0$, which leads to $T(c_2) - T(c_1) > 0$. Therefore, the cut in \textit{case 2} cannot be the optimal model partition in the AI model.

\begin{figure*}[t] 
	\renewcommand{\figurename}{Fig.}
	\centering
	\includegraphics[width=0.90\textwidth]{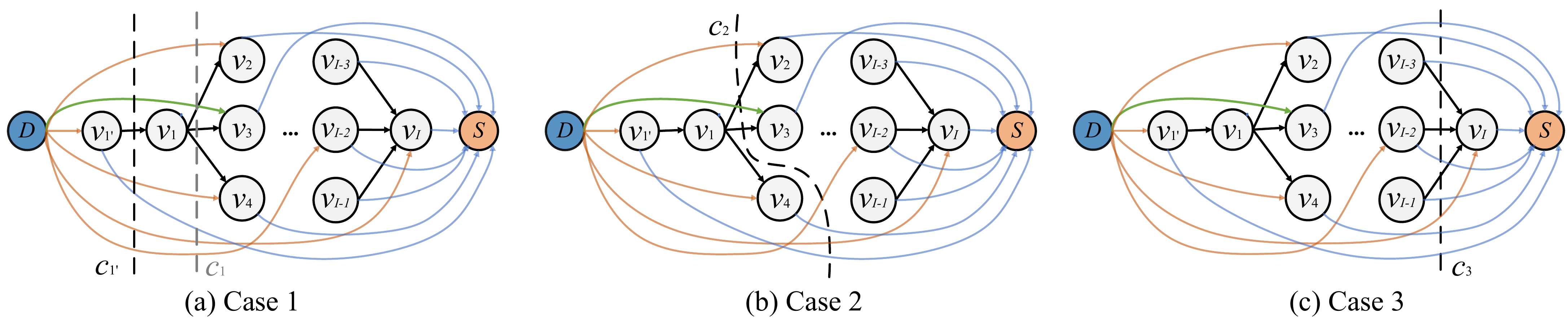}
	\caption{{Three cut cases in the DAG, where each cut’s value is the sum of the weights of the edges that the cut intersects.}}
	\label{fig: cuts_cases in DAG-2}
\end{figure*}

We then examine the corresponding cut in the DAG $\mathcal{G}'$ and find that the same conclusion holds, i.e., the cut in \textit{case 2} is not the minimum \textit{s-t} cut in $\mathcal{G}'$. For example, as illustrated in Fig.~\ref{fig: cuts_cases in DAG-2}(a) and Fig.~\ref{fig: cuts_cases in DAG-2}(b), the difference in cut values is
 \begin{equation}\label{equ: case 2 - case 1 in DAG}
  \begin{split}
    T_{\mathcal{G}'}(c_2) - T_{\mathcal{G}'}(c_1) = N_{loc} \bigg( \xi_{D, v_4} - \xi_{S, v_4} + \frac{a_{v_1}}{R_D} \\ + \frac{a_{v_1}}{R_S} +  \frac{a_{v_4}}{R_D}  + \frac{a_{v_4}}{R_S} \bigg) 
     + \frac{k_{v_4}}{R_D} + \frac{k_{v_4}}{R_S}.
  \end{split}
 \end{equation}
Given Assumption 1, we again have $T_{\mathcal{G}'}(c_2) - T_{\mathcal{G}'}(c_1) > 0$, confirming that the cut in \textit{case 2} cannot be the minimum \textit{s-t} cut in $\mathcal{G}'$.

\emph{Case 3:} The model training delay in the AI model is precisely equal to the value of the corresponding cut in $\mathcal{G}'$. For example, the value of cut $c_3$ in Fig.~\ref{fig: cuts_cases in AI model}(c) matches that of cut $c_3$ in Fig.~\ref{fig: cuts_cases in DAG-2}(c), i.e., $T(c_3) =  T_{\mathcal{G}'}(c_{3})$. Since this case exhibits a one-to-one correspondence between the training delay and the cut value, no further analysis is necessary.

In summary, based on the analysis of all three cases, we have completed the proof.

\section*{Appendix B\\ Proof of Theorem 2}
\renewcommand{\theequation}{B.\arabic{equation}}
\setcounter{equation}{0} 

{For a block, let $T(c^{\min}_{B}) $ and $T(c^{\text{in}}_{B})$ denote the overall training delay under cuts $c_B^{\min}$ and $c^{\text{in}}_B$, respectively. The difference between them is given by}
\begin{equation}\label{equ: the difference between two cut}
     \begin{split}
        T(c^{\min}_{B}) -T(c^{\text{in}}_{B}) =  N_{loc} \frac{a_B^{\min} - a^{\text{in}}_B}{R_{D}} + N_{loc} \frac{a_B^{\min} - a^{\text{in}}_B}{R_{S}}  + \\  \frac{ \sum_{i \in \mathcal{V}_{D}} k_{v_i}}{R_D} 
         + \frac{\sum_{i \in \mathcal{V}_{D}} k_{v_i}}{R_S} + \sum_{i \in \mathcal{V}_{D}} N_{loc} \left( \xi_{D, v_i} - \xi_{S, v_i} \right).
     \end{split}
\end{equation}
{Here, $\mathcal{V}_{D}$ is the set of layers within the block that are assigned to the edge device after the minimum \textit{s-t} cut, and $\mathcal{V}_{D} = \{v_1, v_2,v_3, v_6\}$ in Fig.~4.  $\xi_{D,v_i}$ and $\xi_{S, v_i}$ are the computation delay of processing layer $v_i$ in edge device and server, resepectively. In the right-hand side of Eq.~(\ref{equ: the difference between two cut}), the first term is the difference in transmission delay for the smashed data. The second term is the difference in the gradient's transmission delay. Note that the gradient size during AI model training is equivalent to the size of the smashed data. The third and fourth terms are the differences in transmission delay for the device-side model uploading and downloading, respectively. The fifth term is the difference in the computation delays across the layers.}

{Since the computing capability of the edge server is typically stronger than that of the edge device under Assumption 1, we have $\xi_{D,v_i} - \xi_{S,v_i} > 0$. Therefore, the last term in Eq.~(\ref{equ: the difference between two cut}) is always positive. In addition, the third and fourth terms in Eq.~(\ref{equ: the difference between two cut}) are also positive. If $a_B^{\min} - a_B^{\text{in}} > 0$, it follows that $T(c^{\min}_{B}) -T(c^{\text{in}}_{B}) > 0$. This implies that the overall training delay under the minimum \textit{s-t} cut $c_B^{\min}$ is larger than that under the cut $c_B^{\text{in}}$. In this case, the minimum \textit{s-t} cut in the block is not optimal, and the optimal cut does not intersect any intra-block edges.}

\bibliographystyle{IEEEtran}
\bibliography{IEEEabrv,main}
\end{document}